\documentclass{article}

\usepackage[final]{neurips_2021}

\usepackage{microtype}
\usepackage{graphicx}
\usepackage{subfigure}
\usepackage{booktabs} 

\usepackage[colorlinks=true,citecolor=cyan,linkcolor=blue]{hyperref}

\usepackage{dblfloatfix}
\usepackage{amsmath,amssymb,amsthm}
\usepackage{mathrsfs}
\usepackage{thm-restate}
\usepackage{enumitem}
\usepackage{natbib}
\usepackage{graphicx}
\usepackage{booktabs}
\usepackage{xcolor}
\usepackage{bbm}
\usepackage{verbatim}
\usepackage{wrapfig}

\def\equationautorefname~#1\null{Equation~(#1)\null}

\newtheoremstyle{definition}
{3pt} 
{5pt} 
{} 
{} 
{\bfseries} 
{.} 
{.5em} 
{} 

\theoremstyle{definition}

\newtheorem{theorem}{Theorem}[section]
\newtheorem{proposition}[theorem]{Proposition}
\newtheorem{corollary}[theorem]{Corollary}
\newtheorem{definition}[theorem]{Definition}

\newtheorem{lemma}[theorem]{Lemma}
\newtheorem{example}[theorem]{Example}

\newtheorem{remark}[theorem]{Remark}

\declaretheoremstyle[%
spaceabove=3pt,%
spacebelow=6pt,%
headfont=\normalfont\itshape,%
postheadspace=1em,%
qed=\qedsymbol%
]{mystyle} 
\declaretheorem[name={Proof},style=mystyle,unnumbered,
]{prf}

\usepackage{thm-restate}

\DeclareMathOperator{\arctantwo}{arctan2}

\newcommand{\wassersteinoperator}{T_K}

\newcommand{\cA}{\mathcal{A}}
\newcommand{\cD}{\mathcal{D}}
\newcommand{\cL}{\mathcal{L}}
\newcommand{\cX}{\mathcal{X}}

\title{MICo: Improved representations via sampling-based state similarity for Markov decision processes}

\author{
  Pablo Samuel Castro\thanks{Equal contribution. Correspondence to Pablo Samuel Castro: psc@google.com.} \\
  \small{Google Research, Brain Team}
  \And
  Tyler Kastner$^*$ \\
  \small{McGill University}
  \And
  Prakash Panangaden \\
  \small{McGill University}
  \And
  Mark Rowland \\
  \small{DeepMind}
}

\begin{document}

\maketitle

\begin{abstract}
  We present a new behavioural distance over the state space of a Markov decision
process, and demonstrate the use of this distance as an effective means of
shaping the learnt representations of deep reinforcement learning agents.
While existing notions of state similarity are typically difficult to learn at
scale due to high computational cost and lack of sample-based algorithms, our
newly-proposed distance addresses both of these issues. In addition to
providing detailed theoretical analysis, we provide empirical evidence that
learning this distance alongside the value function yields structured and
informative representations, including strong results on the Arcade Learning
Environment benchmark.

\end{abstract}

\section{Introduction}

The success of reinforcement learning (RL) algorithms in large-scale, complex tasks depends on forming useful representations of the environment with which the algorithms interact. Feature selection and feature learning has long been an important subdomain of RL, and with the advent of deep reinforcement learning there has been much recent interest in understanding and improving the representations learnt by RL agents.

Much of the work in representation learning has taken place from the perspective of \emph{auxiliary tasks} \citep{jaderberg2016reinforcement,bellemare2017distributional,fedus2019hyperbolic}; in addition to the primary reinforcement learning task, the agent may attempt to predict and control additional aspects of the environment. 
Auxiliary tasks shape the agent's representation of the environment \emph{implicitly}, typically via gradient descent on the additional learning objectives. As such, while auxiliary tasks continue to play an important role in improving the performance of deep RL algorithms, our understanding of the effects of auxiliary tasks on representations in RL is still in its infancy.

In contrast to the implicit representation shaping of auxiliary tasks, a
separate line of work on \emph{behavioural metrics}, such as bisimulation metrics
\citep{Desharnais99b,Desharnais04,Ferns04,ferns06methods}, aims to capture
structure in the environment by learning a metric measuring behavioral similarity
between states.  Recent works have successfully
used behavioural metrics to shape the representations of deep RL agents
\citep{gelada2019deepmdp,zhang2021invariant,agarwal2021contrastive}. However, in
practice behavioural metrics are difficult to estimate from both statistical
and computational perspectives, and these works either rely on specific
assumptions about transition dynamics to make the estimation
tractable, and as such can only be applied to limited classes of environments,
or are applied to more general classes of environments not covered by
theoretical guarantees.

The principal objective of this work is to develop new measures of behavioral
similarity that avoid the statistical and computational difficulties described
above, and simultaneously capture richer information about the environment. We
introduce the \emph{MICo ({\bf M}atching under {\bf I}ndependent {\bf Co}uplings) distance}, and develop the theory around its
computation and estimation, making comparisons with existing metrics on the
basis of computational and statistical efficiency.
We demonstrate the usefulness of the representations that MICo yields, both through empirical evaluations in small problems (where we can compute them exactly) as well as in two large benchmark suites: (1) the Arcade Learning Environment \citep{bellemare13arcade,machado2018revisiting}, in which the performance of a wide variety of existing value-based deep RL agents is improved by directly shaping representations via the MICo distance (see \autoref{fig:humanNormalizedIQM}, left), and (2) the DM-Control suite \citep{tassa18dmcontrol}, in which we demonstrate it can improve the performance of both Soft Actor-Critic \citep{haarnoja18sac} and the recently introduced DBC \citep{zhang2021invariant} (see \autoref{fig:humanNormalizedIQM}, right).

\section{Background}
\label{sec:background}
\begin{figure*}[!t]
	\centering
	\includegraphics[keepaspectratio,width=0.45\textwidth]{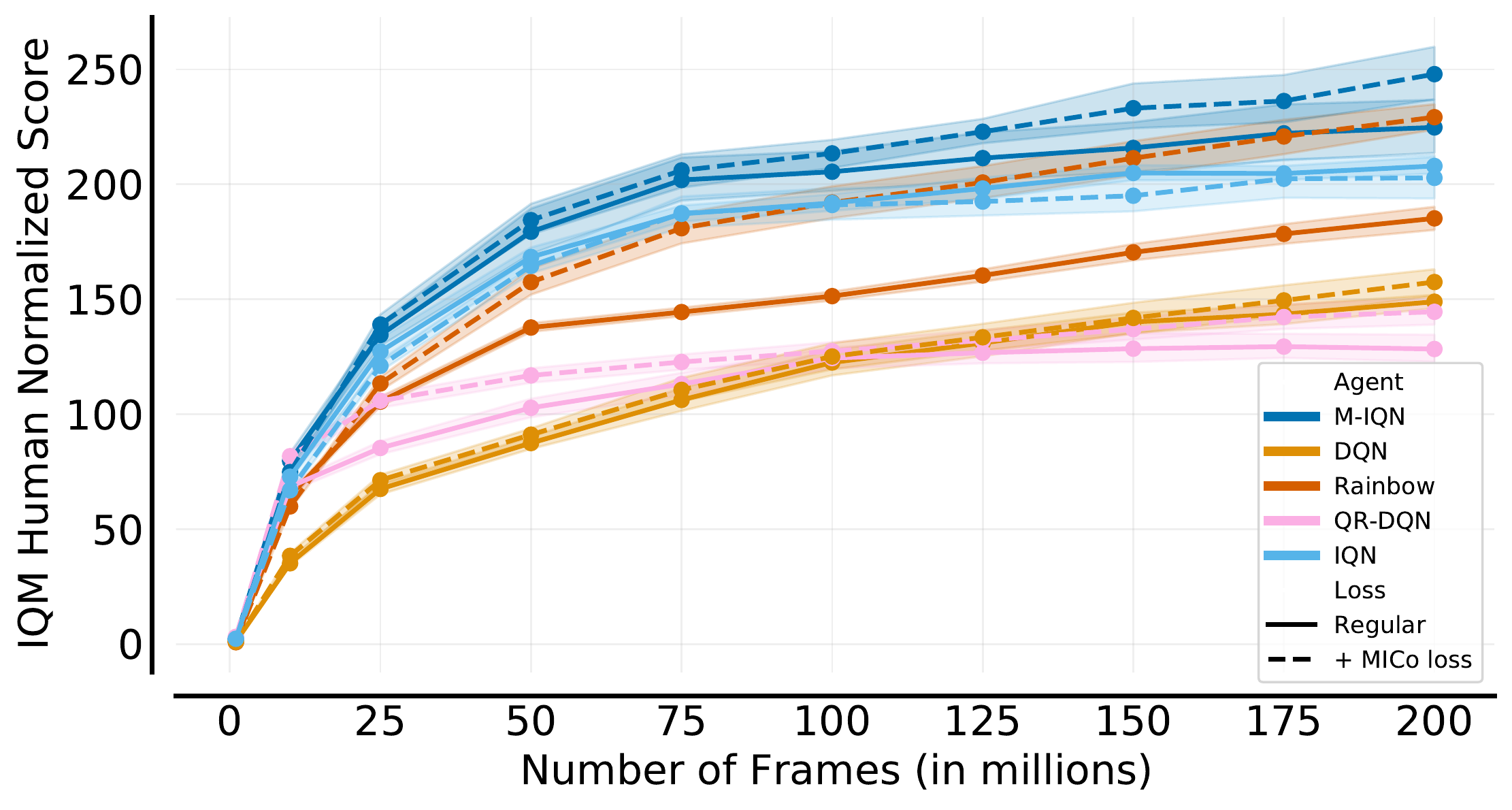}
	\includegraphics[keepaspectratio,width=0.45\textwidth]{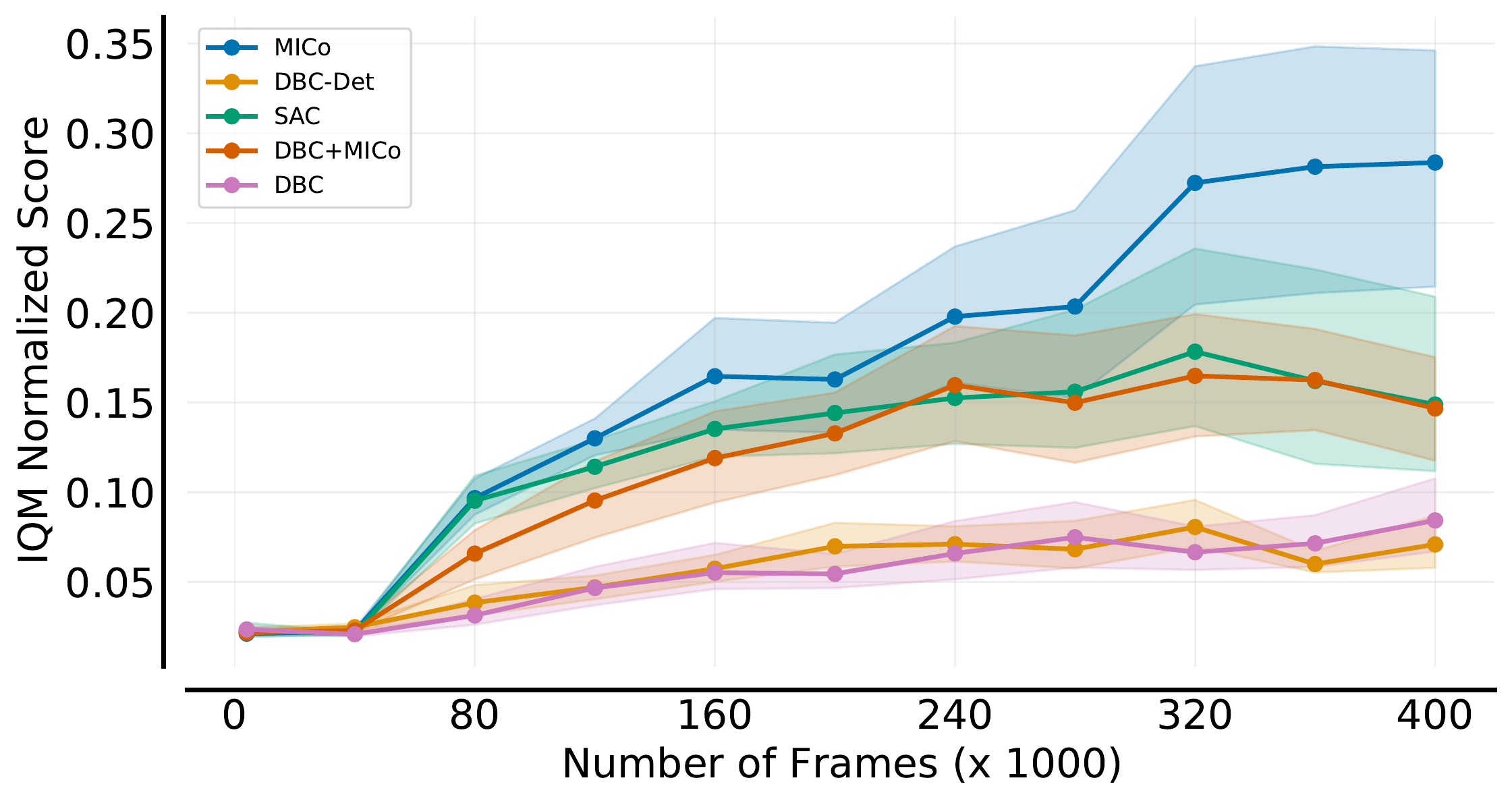}
  \caption{Interquantile Mean human normalized scores of all the agents and losses on the ALE suite (left) and on the DM-Control suite (right), both run with five independent seeds for each agent and environment. In both suites MICo provides a clear advantage.}
  \vspace{-0.5cm}
	\label{fig:humanNormalizedIQM}
\end{figure*}

Before describing the details of our contributions, 
we give a brief overview of the required background in reinforcement learning and bisimulation. We provide more extensive background in \autoref{sec:extendedBackground}.

{\bf Reinforcement learning. } We consider a Markov decision process $(\mathcal{X}, \mathcal{A}, \gamma, P,
r)$ defined by a finite state space $\mathcal{X}$, finite action space
$\mathcal{A}$, transition kernel $P : \mathcal{X} \times \mathcal{A}
\rightarrow \mathscr{P}(\cX)$, reward function $r : \mathcal{X} \times
\mathcal{A} \rightarrow \mathbb{R}$, and discount factor $\gamma \in [0,1)$.
For notational convenience we will write $P_x^a$ and $r_x^a$ for transitions
and rewards, respectively. Policies are mappings from states to distributions
over actions: $\pi \in \mathscr{P}(\mathcal{A})^\mathcal{X}$ and induce a {\em value
function} $V^{\pi}:\cX\rightarrow\mathbb{R}$ defined via the recurrence:
$V^{\pi}(x) := \mathbb{E}_{a\sim\pi(x)}\left[ r_x^a + \gamma\mathbb{E}_{x'\sim
P_x^a}[V^{\pi}(x')]\right]$. In RL we are concerned with finding the optimal
policy $\pi^* = \arg\max_{\pi\in\mathscr{P}(\mathcal{A})^\mathcal{X}}V^{\pi}$ from interaction with sample trajectories with an MDP, {\em without knowledge of $P$ or $r$} (and sometimes not even $\cX$), and the optimal value function
$V^*$ induced by $\pi^*$.

{\bf State similarity and bisimulation metrics.} Various notions of similarity between states in MDPs have been considered in the RL literature, with applications in policy transfer, state aggregation, and representation learning. The \emph{bisimulation metric} \citep{Ferns04} is of particular relevance for this paper, and defines state similarity in an MDP by declaring two states $x,y \in \mathcal{X}$ to be close if their immediate rewards are similar, and the transition dynamics at each state leads to next states which are also judged to be similar. This self-referential notion is mathematically formalised by defining the bisimulation metric $d^\sim$ as the unique fixed-point of the operator $T_K : \mathcal{M}(X) \rightarrow \mathcal{M}(X)$, where $\mathcal{M}(X) = \{ d \in [0,\infty)^{\mathcal{X} \times \mathcal{X}} : d \text{ symmetric and satisfies the triangle inequality} \}$ is the set of pseudometrics on $\mathcal{X}$, given by $T_K(d)(x,y) = \max_{a\in\cA}[ |r^a_x - r^a_{y}| + \gamma W_d(P^a_x, P^a_{y})]$.
Here, $W_d$ is the Kantorovich distance (also known as the Wasserstein distance) over the set of distributions $\mathscr{P}(\mathcal{X})$ with base distance $d$, defined by
$W_d(\mu, \nu) = \inf_{X \sim \mu, Y \sim \nu}\mathbb{E}[d(X,Y)]$, for all $\mu, \nu \in \mathscr{P}(\mathcal{X})$, 
where the infimum is taken over all couplings of $(X, Y)$ with the prescribed marginals \citep{villani08optimal}.

The mapping $T_K$ is a $\gamma$-contraction on $\mathcal{M}(X)$ under the $L^\infty$ norm \citep{ferns2011bisimulation}, and thus by standard contraction mapping arguments analogous to those used to study value iteration, it has a unique fixed point, the bisimulation metric $d^\sim$. \citet{Ferns04} show that this metric bounds differences in the optimal
value function, hence its importance in RL:
\begin{align}\label{eqn:bisimUpperBound}
	|V^*(x) - V^*(y)| \leq d^\sim(x, y) \quad	\forall x,y\in\cX \, .
\end{align}

{\bf Representation learning in RL.} In large-scale environments, it is infeasible to express value functions directly as vectors in $\mathbb{R}^{\mathcal{X} \times \mathcal{A}}$. Instead, RL agents must approximate value functions in a more concise manner, by forming a \emph{representation} of the environment, that is, a feature embedding $\phi : \mathcal{X} \rightarrow \mathbb{R}^M$, and predicting state-action values linearly from these features. \emph{Representation learning} is the problem of finding a useful representation $\phi$.
Increasingly, deep RL agents are equipped with additional losses to aid representation learning. A common approach is to require the agent to make additional predictions (so-called \emph{auxilliary tasks}) with its representation, typically with the aid of extra network parameters, with the intuition that an agent is more likely to learn useful features if it is required to solve many related tasks. We refer to such methods as \emph{implicit} representation shaping, since improved representations are a side-effect of learning to solve auxiliary tasks.

Since bisimulation metrics capture additional information about the MDP in addition to that summarised in value functions, bisimulation metrics are a natural candidate for auxiliary tasks in deep reinforcement learning. \citet{gelada2019deepmdp}, \citet{agarwal2021contrastive}, and \citet{zhang2021invariant} introduce auxiliary tasks based on bisimulation metrics, but require additional assumptions on the underlying MDP in order for the metric to be learnt correctly (Lipschitz continuity, deterministic, and Gaussian transitions, respectively). The success of these approaches provides motivation in this paper to introduce a notion of state similarity applicable to arbitrary MDPs, without further restriction. Further, we learn this state similarity {\em explicitly}: that is, without the aid of any additional network parameters.

\section{Advantages and limitations of the bisimulation metric}\label{sec:bis}
The bisimulation metric $d^\sim$ is a strong notion of distance on the state space of an MDP; it is 
useful in policy transfer through its bound
on optimal value functions \citep{castro10using} and because it is so stringent, it gives good
guarantees for state aggregations \citep{Ferns04,li2006towards}.  However, it has been difficult to use at
scale and compute online, for a variety of reasons that we summarize below.

(i) \textbf{Computational complexity.} The metric can be computed via fixed-point
iteration since the operator $T_K$ is a contraction mapping. The map $\wassersteinoperator$ contracts at rate
$\gamma$ with respect to the $L^\infty$ norm on $\mathcal{M}$, and therefore
obtaining an $\varepsilon$-approximation of $d^\sim$ under this norm requires
$O(\log(1/\varepsilon) / \log(1/\gamma))$ applications of $\wassersteinoperator$ to
an initial pseudometric $d_0$. The cost of each application of $\wassersteinoperator$
is dominated by the computation of $|\mathcal{X}|^2|\mathcal{A}|$ $W_d$
distances for distributions over $\mathcal{X}$, each costing
$\tilde{O}(|\mathcal{X}|^{2.5})$ in theory \citep{lee2014path}, and
$\tilde{O}(|\mathcal{X}|^3)$ in practice
\citep{pele2009fast,guo2020fast,peyre2019computational}. Thus, the overall
practical cost is
$\tilde{O}(|\mathcal{X}|^{5}|\mathcal{A}|\log(\varepsilon) / \log(\gamma))$.

(ii) \textbf{Bias under sampled transitions.} Computing $\wassersteinoperator$
requires access to the transition probability distributions $P_x^a$ for
each $(x, a) \in \mathcal{X} \times \mathcal{A}$ which, as mentioned in \autoref{sec:background}, are typically not available; instead, stochastic approximations to the operator of interest are employed. Whilst there has been work in studying online, sample-based
approximate computation of the bisimulation metric \citep{ferns06methods,comanici2012fly}, these
methods are generally biased, in contrast to sample-based estimation of standard RL operators.

(iii) \textbf{Lack of connection to non-optimal policies.} One of the principal behavioural characterisations of the bisimulation metric $d^\sim$ is the upper bound shown in \autoref{eqn:bisimUpperBound}.
However, in general we do not have $|V^\pi(x) - V^\pi (y)| \leq d^\sim(x, y)$ for arbitrary policies $\pi \in \Pi$; a simple example is illustrated in \autoref{fig:ex3}. More generally, notions of state similarity that the bisimulation metric encodes may not be closely related to behavioural similarity under an arbitrary policy $\pi$. Thus, learning about $d^\sim$ may not in itself be useful for large-scale reinforcement learning agents.

\begin{wrapfigure}{r}{0.5\textwidth}
	\centering
	\includegraphics[keepaspectratio,width=.4\textwidth]{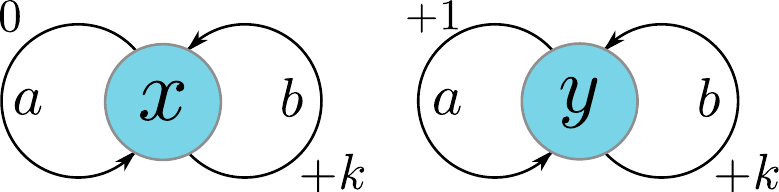}
	\caption{MDP illustrating that the upper bound for any $\pi$ is not generally satisfied. Here, $d^\sim(x, y) = (1-\gamma)^{-1}$, but for $\pi(b|x)=1, \pi(a|y) = 1$, we have $|V^\pi(x) - V^\pi(y)| = k(1-\gamma)^{-1}$.}
	\label{fig:ex3}
\end{wrapfigure}
\vspace{-0.5em}

Property (i) expresses the intrinsic computational difficulty of computing this
metric. Property (ii) illustrates the problems associated with
attempting to move from operator-based computation to online, sampled-based computation of the
metric (for example, when the environment dynamics are unknown). Finally,
property (iii) shows that even if the metric is computable exactly, the
information it yields about the MDP may not be practically useful. Although
$\pi$-bisimulation (introduced by \citet{castro2020scalable} and extended by
\citet{zhang2021invariant}) addresses property (iii), their 
practical algorithms are limited to MDPs with deterministic transitions \citep{castro2020scalable} or
MDPs with Gaussian transition kernels \citep{zhang2021invariant}. Taken together, these three properties motivate the search for a metric without these
shortcomings, which can be used in combination with deep reinforcement learning.

\section{The MICo distance}
We now present a new notion of distance for state similarity, which we refer to
as \emph{MICo} ({\bf M}atching under {\bf I}ndependent {\bf Co}uplings), designed to overcome the drawbacks described above.

Motivated by the drawbacks described in \autoref{sec:bis}, we make several modifications to the operator $\wassersteinoperator$ introduced above: (i) in order to deal with the prohibitive cost of computing the Kantorovich distance, which optimizes over all coupling of the distributions $P_x^a$ and $P_y^a$, we use the independent coupling;
(ii) to deal with lack of connection to non-optimal policies, we consider an
on-policy variant of the metric, pertaining to a chosen policy $\pi \in \mathscr{P}(\mathcal{A})^\mathcal{X}$.  This leads us to the following definition.

\begin{definition}
	Given $\pi \in\mathscr{P}(\mathcal{A})^\mathcal{X}$, 
	the \emph{MICo update operator} $T^\pi_M : \mathbb{R}^{\mathcal{X}\times\mathcal{X}} \rightarrow \mathbb{R}^{\mathcal{X}\times\mathcal{X}}$ is:
	\begin{align}
		\label{eqn:uUpdate}
    (T^\pi_M U)(x, y)  = |r^\pi_x - r^\pi_y| + \gamma \mathbb{E}_{\begin{subarray}{l}x'\sim P^{\pi}_x \\ y'\sim P^{\pi}_y\end{subarray}} \left[ U(x', y') \right]
	\end{align}
	for all $U:\mathcal{X}\times\mathcal{X}\rightarrow\mathbb{R}$, with $r^\pi_x = \sum_{a \in \mathcal{A}} \pi(a|x) r_x^a$ and $P^{\pi}_x = \sum_{a\in\cA}\pi(a|x)P_x^a(\cdot)$ for all $x \in \mathcal{X}$.
\end{definition}

As with the bisimulation operator, this can be thought of as encoding desired
properties of a notion of similarity between states in a self-referential
manner; the similarity of two states $x, y \in \mathcal{X}$ should be determined
by the similarity of the rewards and the similarity of the states they lead to.

\begin{restatable}{proposition}{micoContraction}\label{prop:mico-contraction}
	The operator $T^\pi_M$ is a contraction mapping on $\mathbb{R}^{\mathcal{X}\times\mathcal{X}}$ with respect to the $L^\infty$ norm.
\end{restatable}
\begin{prf}
See \autoref{sec:missingProofs}.
\end{prf}

The following corollary now follows immediately from Banach's fixed-point theorem and the completeness of $\mathbb{R}^{\mathcal{X}\times\mathcal{X}}$ under the $L^\infty$ norm.

\begin{corollary}\label{corr:mico-fp}
	The MICo operator $T^\pi_M$ has a unique fixed point $U^\pi \in \mathbb{R}^{\mathcal{X}\times\mathcal{X}}$, and repeated application of $T^\pi_M$ to any initial function $U \in \mathbb{R}^{\mathcal{X}\times\mathcal{X}}$ converges to $U^\pi$.
\end{corollary}

Having defined a new operator, and shown that it has a corresponding fixed-point, there are two questions to address: Does this new notion of distance overcome the drawbacks of the bisimulation metric described above; and what does this new object tell us about the underlying MDP?

\subsection{Addressing the drawbacks of the bisimulation metric}

We introduced the MICo distance as a means of overcoming some of the shortcomings associated with the bisimulation metric, described in \autoref{sec:bis}. In this section, we provide a series of results that show that the newly-defined notion of distance addressess each of these shortcomings. The proofs of these results rely on the following lemma, connecting the MICo operator to a lifted MDP.
This result is crucial for much of the analysis that follows, so we describe the proof in full detail. 

\begin{lemma}[\textbf{Lifted MDP}]\label{lem:aux-mdp}
	The MICo operator $T^\pi_M$ is the Bellman evaluation operator for an auxiliary MDP.
\end{lemma}

\begin{prf}
	Given the MDP specified by the tuple $(\mathcal{X}, \mathcal{A}, P, R)$, we construct an auxiliary MDP $(\widetilde{\mathcal{X}},\widetilde{\mathcal{A}},
	\widetilde{P}, \widetilde{R})$, by taking the
	state space to be $\widetilde{\mathcal{X}} = \mathcal{X}^2$, the action space
	to be $\widetilde{\mathcal{A}} = \mathcal{A}^2$, the transition dynamics to be
    given by $\widetilde{P}_{(u, v)}^{(a, b)}((x,y)) = P_u^a(x)P_v^b(y)$ for
    all $(x,y), (u,v) \in \mathcal{X}^2$, $a,b \in \mathcal{A}$, and the action-independent rewards to be $\widetilde{R}_{(x,y)} = |r^\pi_x - r^\pi_y|$ for all $x, y \in \mathcal{X}$. The Bellman evaluation operator $\widetilde{T}^{\tilde{\pi}}$ for this auxiliary MDP at discount rate $\gamma$ under the policy $\tilde{\pi}(a,b|x,y) = \pi(a|x) \pi(b|y)$ is given by (for all $U \in \mathbb{R}^{\mathcal{X}\times\mathcal{X}}$ and $(x, y) \in \mathcal{X} \times\mathcal{X}$):
	\begin{align*}
    (\widetilde{T}^{\tilde{\pi}}U)(x,y) & = \widetilde{R}_{(x,y)}\! +\! \gamma\!\!\!\! \sum_{(x^\prime, y^\prime) \in \mathcal{X}^2} \!\!\!\!\!\! \widetilde{P}_{(x, y)}^{(a, b)}((x^\prime, y^\prime)) \tilde{\pi}(a,b|x,y) U(x^\prime, y^\prime) \\
    & = |r^\pi_x - r^\pi_y| + \gamma\!\!\!\! \sum_{(x^\prime, y^\prime) \in \mathcal{X}^2}  P^\pi_x(x^\prime)P_y^\pi(y^\prime)  U(x^\prime, y^\prime) = (T^\pi_MU)(x, y) \, . \qedhere
	\end{align*}
\end{prf}

\begin{remark}
	\citet{ferns2014bisimulation} noted that the bisimulation metric can be interpreted as the optimal value function in a related MDP, and that the functional $\wassersteinoperator$ of $T_K$ can be interpreted as a Bellman optimality operator. However, their proof was non-constructive, the related MDP being characterised via the solution of an optimal transport problem. In contrast, the connection described above is constructive, and will be useful in understanding many of the theoretical properties of MICo. \citet{ferns2014bisimulation} also note that the $W_d$ distance in the definition of $T_K$ can be upper-bounded by taking a restricted class of couplings of the transition distributions. The MICo metric can be viewed as restricting the coupling class precisely to the singleton containing the independent coupling.
\end{remark}

With \autoref{lem:aux-mdp} established, we can now address each of the points (i), (ii), and (iii) from \autoref{sec:bis}.

\textbf{(i) Computational complexity. }
The key result regarding the computational complexity of computing the MICo distance is as follows.

\begin{proposition}[\textbf{MICo computational complexity}]
	The computational complexity of computing an $\varepsilon$-approximation in $L^\infty$ to the MICo metric is $O(|\mathcal{X}|^4 \log(\varepsilon) / \log(\gamma))$.
\end{proposition}
\begin{prf}
	Since, by \autoref{prop:mico-contraction}, the operator $T^\pi_M$ is a $\gamma$-contraction under $L^\infty$, we require $\mathcal{O}(\log(\varepsilon) / \log(\gamma))$ applications of the operator to obtain an $\varepsilon$-approximation in $L^\infty$. Each iteration of value iteration updates $|\mathcal{X}|^2$ table entries, and the cost of each update is $\mathcal{O}(|\mathcal{X}|^2)$, leading to an overall cost of $O(|\mathcal{X}|^4\log(\varepsilon) / \log(\gamma))$.
\end{prf}

In contrast to the bisimulation metric, this represents a computational saving
of $O(|\mathcal{X}|)$, which arises from the lack of a need to
solve optimal transport problems over the state space in computing the MICo
distance. There is a further saving of $\mathcal{O}(|\mathcal{A}|)$ that arises
since MICo focuses on an individual policy $\pi$, and so does not require the
max over actions in the bisimulation operator definition.

\textbf{(ii) Online approximation. }
Due to the interpretation of the MICo operator $T^\pi_M$ as the Bellman evaluation operator in an auxiliary MDP, established in \autoref{lem:aux-mdp}, algorithms and associated proofs of correctness for computing the MICo distance online can be straightforwardly derived from standard online algorithms for policy evaluation. We describe a straightforward approach, based on the TD(0) algorithm, and also note that the wide range of online policy evaluation methods incorporating off-policy corrections and multi-step returns, as well as techniques for applying such methods at scale, may also be used.

Given a current estimate $U_t$ of the fixed point of $T^\pi_M$ and a pair of observations $(x, a, r, x^\prime)$, $(y, b, \tilde{r}, y^\prime)$ generated under $\pi$, we can define a new estimate $U_{t+1}$ via
\begin{align}\label{eq:stoc-update}
  U_{t+1}(x, y) \leftarrow (1-\epsilon_t(x, y))U_t(x, y) + \epsilon_t(x, y) ( |r - \tilde{r}| + \gamma U_{t}(x', y') )\,
\end{align}
and $U_{t+1}(\tilde{x}, \tilde{y}) = U_t(\tilde{x}, \tilde{y})$ for all other state-pairs $(\tilde{x}, \tilde{y}) \not=(x, y)$,
for some sequence of stepsizes $\{\epsilon_t(x, y) \mid t \geq 0, (x, y) \in \mathcal{X}^2 \}$.
Sufficient conditions for convergence of this algorithm can be deduced straightforwardly from corresponding conditions for TD(0). We state one such result below. An important caveat is that the correctness of this particular algorithm depends on rewards depending only on state; one can switch to state-action metrics if this hypothesis is not satisfied.

\begin{restatable}{proposition}{propSA}\label{prop:SA}
	Suppose rewards depend only on state, and consider the sequence of estimates $(U_t)_{t \geq 0}$, with $U_0$ initialised arbitrarily, and $U_{t+1}$ updated from $U_t$ via a pair of transitions $(x_t, a_t, r_t, x'_t)$, $(y_t, b_t, \tilde{r}_t, y'_t)$ as in \autoref{eq:stoc-update}. 
	If all state-pairs tuples are updated infinitely often, and stepsizes for these updates satisfy the Robbins-Monro conditions. Then $U_t \rightarrow U^\pi$ almost surely.
\end{restatable}

\begin{prf}
	Under the assumptions of the proposition, the update described is exactly a TD(0) update in the lifted MDP described in \autoref{lem:aux-mdp}. We can therefore appeal to Proposition~4.5 of \citet{bertsekas1996neuro} to obtain the result.
\end{prf}

Thus, in contrast to the Kantorovich metric, convergence to the exact MICo metric is possible with an online algorithm that uses sampled transitions.

\textbf{(iii) Relationship to underlying policy.}
In contrast to the bisimulation metric, we have the following on-policy guarantee for the MICo metric.

\begin{proposition}
	For any $\pi \in \mathscr{P}(\mathcal{A})^{\mathcal{X}}$ and states $x,y \in \mathcal{X}$, we have $|V^\pi(x) - V^\pi(y)| \leq U^\pi(x, y)$.
  \label{prop:valueFunctionBound}
\end{proposition}
\vspace{-1.5em}
\begin{prf}
	We apply a coinductive argument \citep{kozen07coinductive} to show that if $|V^\pi(x) - V^\pi(y)| \leq U(x, y) \ \text{for all } x, y \in \mathcal{X}$, for some $U \in \mathbb{R}^{\mathcal{X}\times\mathcal{X}}$ symmetric in its two arguments, then we also have
$|V^\pi(x) - V^\pi(y)| \leq (T^\pi_M U)(x, y) \ \text{for all } x, y \in \mathcal{X}$.
	Since the hypothesis holds for the constant function $U(x,y) = 2 \max_{z, a} |r(z, a)|/(1-\gamma)$, and $T^\pi_M$ contracts around $U^\pi$, the conclusion then follows. Therefore, suppose the hypothesis holds. Then we have
	\begin{align*}
    V^\pi(x) - V^\pi(y) & = r^\pi_x - r^\pi_y + \gamma \sum_{x' \in \mathcal{X}} P^\pi_x(x') V(x') - \gamma \sum_{y' \in \mathcal{X}} P^\pi_y(y') V(y') \\
    & \leq |r^\pi_x - r^\pi_y| + \gamma \sum_{x', y' \in \mathcal{X}} P^\pi_x(x')P^\pi_y(y') (V^\pi(x') - V^\pi(y')) \\
    & \leq |r^\pi_x - r^\pi_y| + \gamma \sum_{x', y' \in \mathcal{X}} P^\pi_x(x')P^\pi_y(y') U(x', y') = (T^\pi_M U)(x, y) \, .
	\end{align*}
	By symmetry, $V^\pi(y) - V^\pi(x) \leq (T^\pi_M U)(x, y)$, as required.
\end{prf}

\subsection{Diffuse metrics}
To characterize the nature of the fixed point $U^\pi$, we introduce the notion of a \emph{diffuse metric}. 
\begin{definition}
    Given a set $\mathcal{X}$, a function $d:\mathcal{X}\times 
    \mathcal{X} \to \mathbb{R}$ is a \emph{diffuse metric} if the following axioms hold:
        (i) $d(x,y)\geq 0$ for any $x,y\in \mathcal{X}$; 
        (ii) $d(x,y)=d(y,x)$ for any $x,y\in \mathcal{X}$;
        (iii) $d(x,y)\leq d(x,z)+d(y,z)$ $\forall x,y,z\in \mathcal{X}$.
\end{definition}
These differ from the standard metric axioms in the first point: we no longer
require that a point has zero self-distance, and two distinct points
may have zero distance.  Notions of this kind are increasingly common in machine
learning as researchers develop more computationally tractable versions of
distances, as with entropy-regularised optimal transport distances
\citep{cuturi2013sinkhorn}, which also do not satisfy the axiom of zero
self-distance.

An example of a diffuse metric is the Łukaszyk–Karmowski distance
\citep{LKmetric}, which is used in the MICo metric as the operator between the
next-state distributions.  Given a diffuse metric space $(\mathcal{X}, \rho)$, the Łukaszyk–Karmowski distance $d^{\rho}_{\text{LK}}$ is a diffuse metric on probability measures on $\mathcal{X}$ given by $d^\rho_{\text{LK}}(\mu,\nu)=\mathbb{E}_{x\sim \mu, y\sim \nu}[\rho(x,y)]$.
This example demonstrates the origin of the name \emph{diffuse} metrics: the
non-zero self distances arises from a point being spread across a probability
distribution.
In terms of the Łukaszyk–Karmowski distance, the MICo distance can be written as the fixed point $U^\pi(x,y)=|r^\pi_x-r^\pi_y|+d_{\text{LK}}(U^\pi) (P^\pi_x,P^\pi_y)$.
This characterisation leads to the following result.

\begin{restatable}{proposition}{micoDiffuse}\label{prop:micoDiffuse}
The MICo distance is a diffuse metric. 
\end{restatable}
\begin{prf}
	Non-negativity and symmetry of $U^\pi$ are clear, so it remains to check the triangle inequality. To do this, we define a sequence of iterates $(U_k)_{k \geq 0}$ in $ \mathbb{R}^{\mathcal{X}\times\mathcal{X}}$ by $U_0(x, y) = 0$ for all $x, y \in \mathcal{X}$, and $ U_{k+1} = T^\pi_M U_k$ for each $k \geq 0$. Recall that by \autoref{corr:mico-fp} that $U_k \rightarrow U^\pi$. We will show that each $U_k$ satisfies the triangle inequality by induction. By taking limits on either side of the inequality, we will then recover that $U^\pi$ itself satisfies the triangle inequality.
    The base case of the inductive argument is clear from the choice of $U_0$. For the inductive step, assume that for some $k \geq 0$, $U_k(x,y) \leq U_k(x, z) + U_k(z, y)$ for all $x, y, z \in \mathcal{X}$. Now for any $x, y, z \in \mathcal{X}$, we have
	\begin{align*}
		U_{k+1}(x, y) & = |r^\pi_x - r^\pi_y| + \gamma \mathbb{E}_{X' \sim P^\pi_x, Y' \sim P^\pi_y}[U_k(X', Y')] \\
		& \leq |r^\pi_x - r^\pi_z| + |r^\pi_z - r^\pi_y| + \gamma \mathbb{E}_{X' \sim P^\pi_x, Y' \sim P^\pi_y, Z' \sim P^\pi_z}[U_k(X', Z') + U_k(Z', Y')] \\
		& = U_{k+1}(x, z) + U_{k+1}(z, y) \, . \qedhere
	\end{align*}
\end{prf}

It is interesting to note that a state $x \in \mathcal{X}$ has
zero self-distance iff the Markov chain induced by $\pi$ initialised at $x$ is deterministic, and the
magnitude of a state's self-distance is indicative of the amount of
``dispersion'' 
in the distribution. Hence, in general, we have $U^\pi(x, x) > 0$, and $U^\pi(x, x) \not= U^\pi(y, y)$ for distinct states $x, y \in \mathcal{X}$. See the appendix for further discussion of diffuse metrics and related constructions.

\section{The MICo loss}
\label{sec:loss}
The impetus of our work is the development of principled mechanisms for directly shaping the representations used by RL agents so as to improve their learning. In this section we present a novel loss based on the MICo update operator $T^\pi_M$ given in \autoref{eqn:uUpdate} that can be incorporated into any RL agent. Given the fact that MICo is a diffuse metric that can admit non-zero self-distances, special care needs to be taken in how these distances are learnt; indeed, traditional mechanisms for measuring distances between representations (e.g. Euclidean and cosine distances) are geometrically-based and enforce zero self-distances.

\newcommand{\nnparam}{\xi}

We assume an RL agent learning an estimate $Q_{\nnparam,\omega}$ defined by the composition of two function approximators $\psi$ and $\phi$ with parameters $\xi$ and $\omega$, respectively: $Q_{\nnparam, \omega}(x, \cdot) = \psi_{\nnparam}(\phi_{\omega}(x))$ (note that this can be the critic in an actor-critic algorithm such as SAC). We will refer to $\phi_{\omega}(x)$ as the {\em representation} of state $x$ and aim to make distances between representations match the MICo distance; we refer to $\psi_{\nnparam}$ as the {\em value approximator}.
We define the parameterized representation distance, $U_{\omega}$, as an approximant to $U^{\pi}$:
\[ U^{\pi}(x, y) \approx U_{\omega}(x, y)  := \frac{\| \phi_{\omega}(x) \|^2_2 + \| \phi_{\omega}(y) \|^2_2 }{2} + \beta \theta(\phi_{\omega}(x), \phi_{\omega}(y)) \]
where $\theta(\phi_\omega(x), \phi_\omega(y))$ is the angle between vectors $\phi_\omega(x)$ and $\phi_\omega(y)$ and $\beta$ is a scalar (in our results we use $\beta=0.1$ but present results with other values of $\beta$ in the appendix).

\begin{figure}[!t]
	\centering
	\includegraphics[keepaspectratio,width=.4\textwidth]{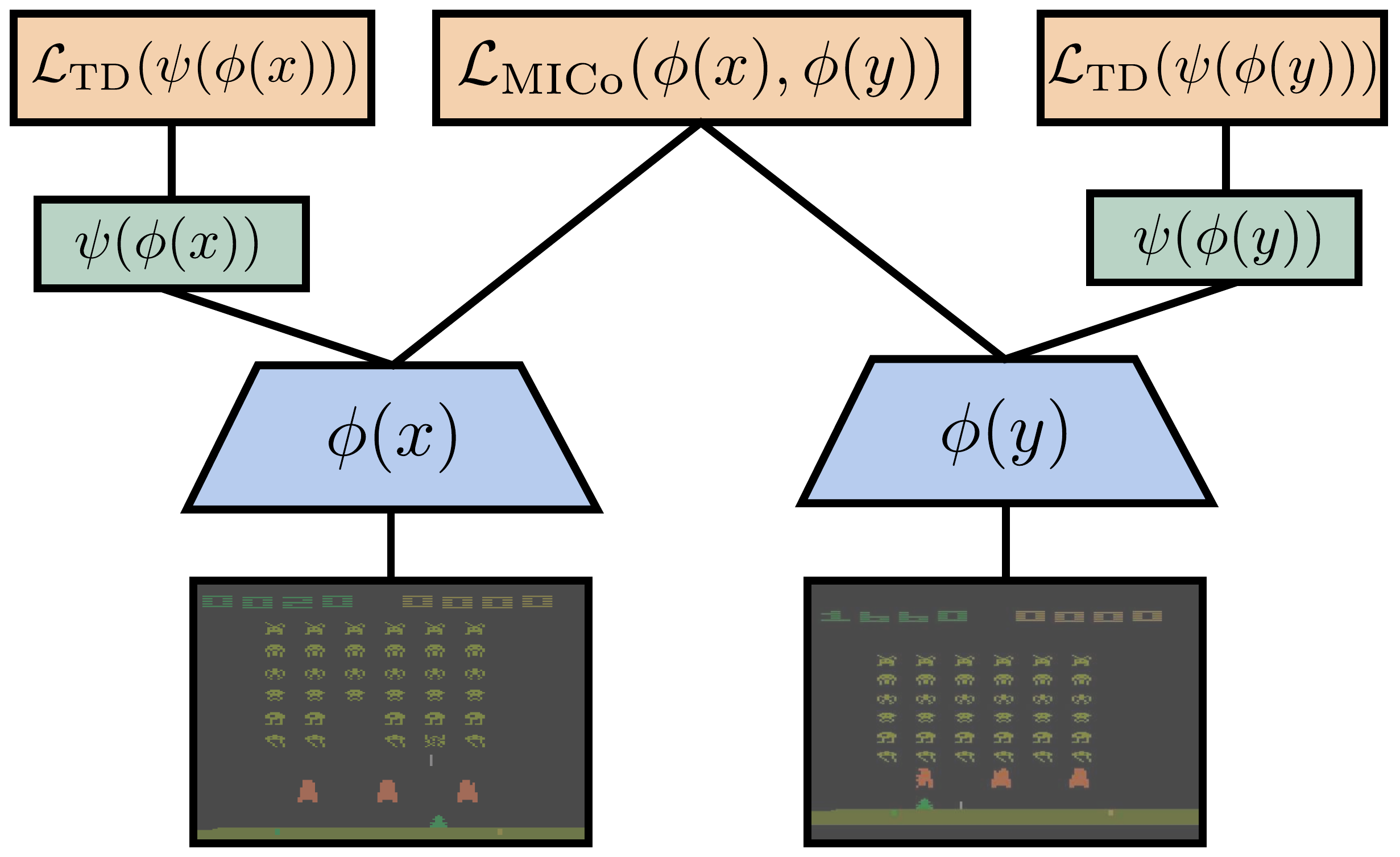}
	\includegraphics[keepaspectratio,width=.4\textwidth]{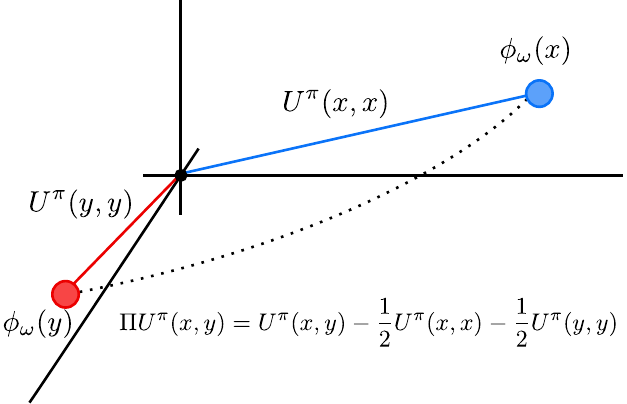}
	\caption{{\bf Left:} Illustration of network architecture for learning MICo; {\bf Right:} The projection of MICo distances onto representation space.}
	\label{fig:network}
  \vspace{-0.5cm}
\end{figure}
Based on \autoref{eqn:uUpdate}, our learning target is then $T^U_{\bar{\omega}}(r_x, x', r_y, y') = |r_x - r_y| + \gamma U_{\bar{\omega}}(x', y')$,
where $\bar{\omega}$ is a separate copy of the network parameters that are synchronised with $\omega$ at infrequent intervals. This is a common practice that was introduced by \citet{mnih15human} (and in fact, we use the same update schedule they propose). The loss for this learning target is
\begin{align*}
  \mathcal{L}_{\text{MICo}}(\omega) = & \mathbb{E}_{\langle x, r_x, x'\rangle ,  \langle y, r_y, y'\rangle}\left[ \left(T^U_{\bar{\omega}}(r_x, x', r_y, y') - U_{\omega}(x, y)\right)^2 \right]
\end{align*}
where $\langle x, r_x, x'\rangle$ and $\langle y, r_y, y'\rangle$ are pairs of transitions sampled from the agent's replay buffer.
We can combine $\mathcal{L}_{\text{MICo}}$ with the temporal-difference loss $\mathcal{L}_{\text{TD}}$ of any RL agent as $(1-\alpha)\mathcal{L}_{\text{TD}} + \alpha\mathcal{L}_{\text{MICo}}$, where $\alpha \in (0,1)$. Each sampled mini-batch is used for both MICo and TD losses. \autoref{fig:network} (left) illustrates the network architecture used for learning.

Although the loss $\mathcal{L}_{\text{MICo}}$ is designed to learn the MICo diffuse metric $U^\pi$, the values of the metric itself are parametrised through $U_\omega$ defined above, which is constituted by several distinct terms. This appears to leave a question as to how the representations $\phi_\omega(x)$ and $\phi_\omega(y)$, as Euclidean vectors, are related to one another when the MICo loss is minimised. Careful inspection of the form of $U_\omega(x, y)$ shows that the (scaled) angular distance between $\phi_\omega(x)$ and $\phi_\omega(y)$ can be recovered from $U_\omega$ by subtracting the learnt approximations to the self-distances $U^\pi(x, x)$ and $U^\pi(y,y)$ (see \autoref{fig:network}, right). We therefore define the reduced MICo distance $\Pi U^\pi$, which encodes the distances enforced between the representation vectors $\phi_\omega(x)$ and $\phi_\omega(y)$, by:

\vspace{-1.5em}
\begin{align*}
	\beta \theta(\phi_{\omega}(x) ,\phi_{\omega}(y))  \approx \Pi U^{\pi}(x, y) = U^{\pi}(x, y) - \frac{1}{2}U^{\pi}(x, x) - \frac{1}{2}U^{\pi}(y, y)  \, .
\end{align*}

In the following section we investigate the following two questions: {\bf (1)} How informative of $V^{\pi}$ is $\Pi U^{\pi}$?; and {\bf (2)} How useful are the features encountered by $\Pi U^{\pi}$ for policy evaluation? We conduct these investigations on tabular environments where we can compute the metrics exactly, which helps clarify the behaviour of our loss when combined with deep networks in \autoref{sec:empirical}.

\subsection{Value bound gaps}

Although \autoref{prop:valueFunctionBound} states that we have $|V^{\pi}(x) - V^{\pi}(y)| \leq U^{\pi}(x, y)$, we do not, in general, have the same upper bound for $\Pi U^{\pi}(x, y)$ as demonstrated by the following result.

\begin{lemma}
There exists an MDP with $x,y\in\cX$, and $\pi\in\Pi$ where $|V^{\pi}(x) - V^{\pi}(y)| > \Pi U^{\pi}(x, y)$.
\end{lemma}

\begin{prf}
  Consider a single-action MDP with two states ($x$ and $y$) where $y$ is absorbing, $x$ transitions with equal probability to $x$ and $y$, and a reward of $1$ is received only upon taking an action from state $x$. There is only one policy for this MDP which yields the value function $V(x) \approx 1.8$ and $V(y) = 0$. The MICo distance gives $U(x, x) \approx 1.06$, $U(x, y) \approx 1.82$, and $U(y, y) = 0$, while the reduced MICo distance yields $\Pi U(x, x) = \Pi U(y, y) = 0$, and $\Pi U(x, y) \approx 1.29 < |V(x) - V(y)| = 1.8$.
\end{prf}

Despite this negative result, it is worth evaluating how often {\em in practice} this inequality is violated and by how much, as this directly impacts the utility of this distance for learning representations.

To do so, we make use of \emph{Garnet MDPs}, a class of randomly generated MDPs \citep{archibald95generation,piot14difference}. Given a specified number of states $n_{\cX}$ and the number of actions $n_{\cA}$, $\text{Garnet}(n_{\cX}, n_{\cA})$ is generated as follows: {\bf 1.} The branching factor $b_{x, a}$ of each transition $P_x^a$ is sampled uniformly from $[1:n_{\cX}]$. {\bf 2.} $b_{x, a}$ states are picked uniformly randomly from $\cX$ and assigned a random value in $[0, 1]$; these values are then normalized to produce a proper distribution $P_x^a$.\\{\bf 3.} Each $r_x^a$ is sampled uniformly in $[0, 1]$.

For each $\text{Garnet}(n_{\cX}, n_{\cA})$ we sample 100 stochastic policies $\{\pi_i\}$ and compute the average gap: $\frac{1}{100 |\cX|^2}\sum_i \sum_{x, y} d(x, y) - |V^{\pi_i}(x) - V^{\pi_i}(y)|$, where $d$ stands for any of the considered metrics. Note we are measuring the {\em signed} difference, as we are interested in the frequency with which the upper bound is violated. As seen in \autoref{fig:valueBoundGap} (left), our metric {\em does} on average provide an upper bound on the difference in values that is also tighter bound than those provided by $U^{\pi}$ and $\pi$-bisimulation. This suggests that the resulting representations remain informative of value similarities.

\begin{figure}[!t]
	\centering
	\includegraphics[keepaspectratio,width=.4\textwidth]{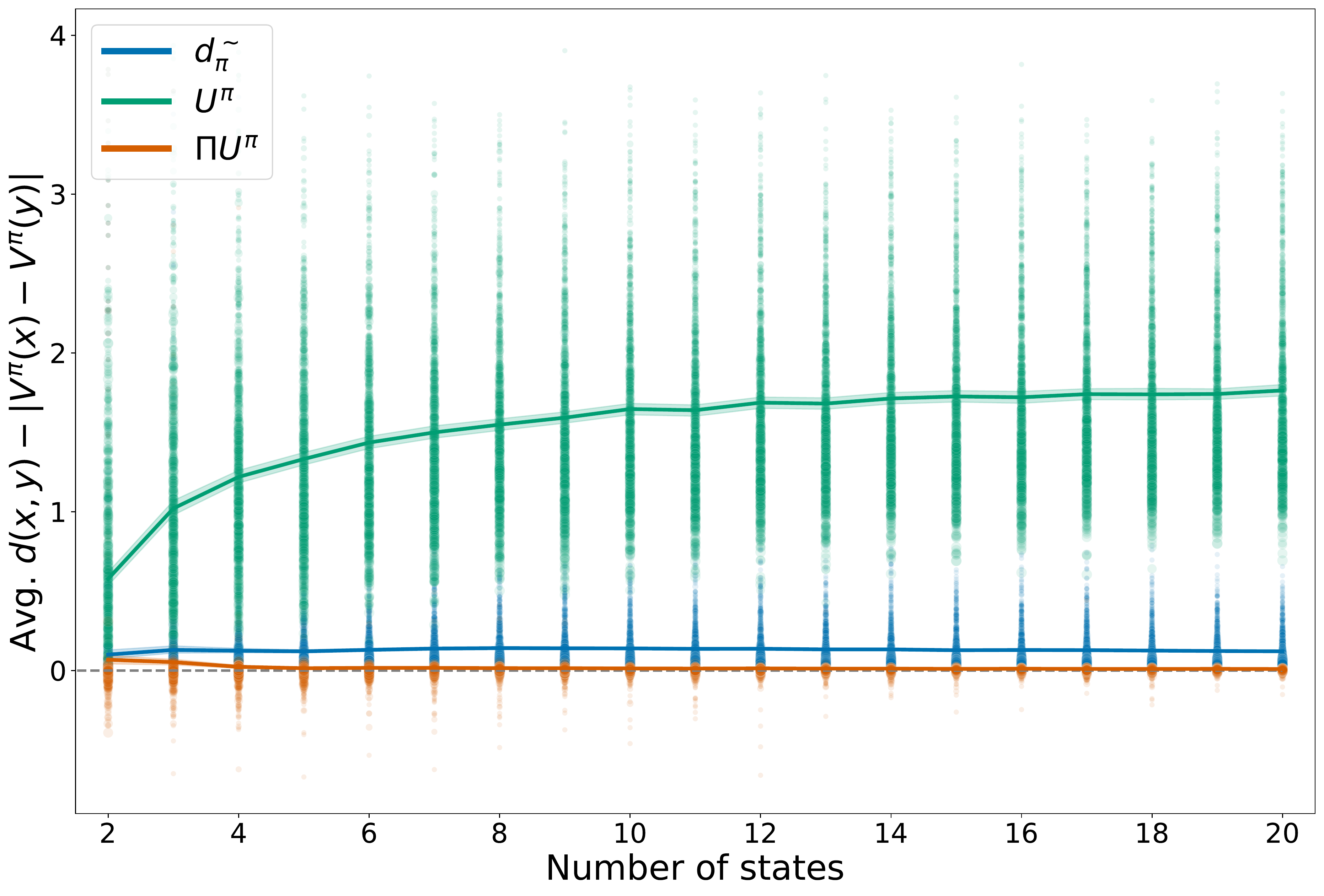}
	\includegraphics[keepaspectratio,width=.4\textwidth]{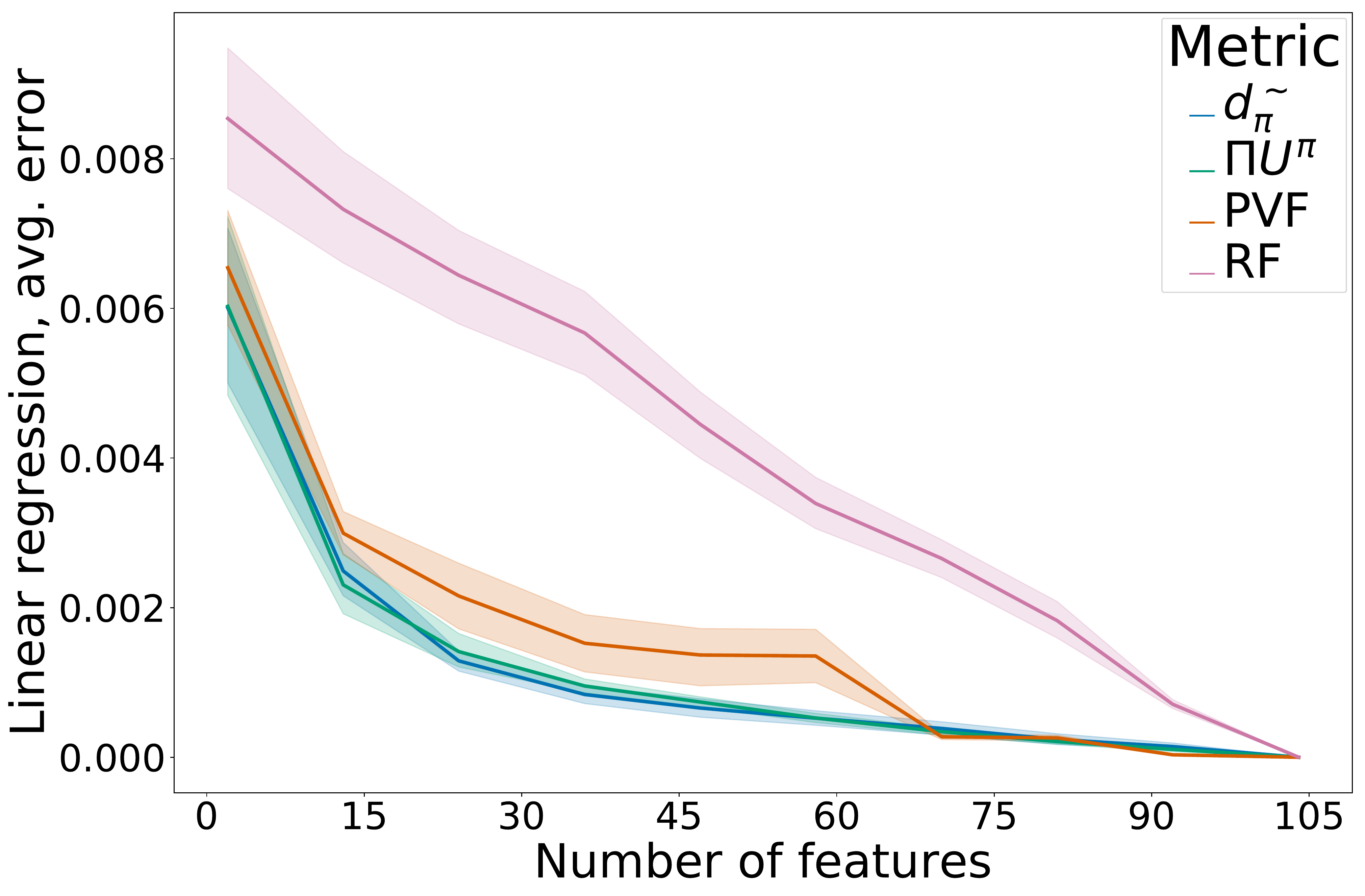}
	\caption{{\bf Left:} The gap between the difference in values and the various distances for Garnet MDPs with varying numbers of actions (represented by circle sizes); {\bf Right:} Average error when performing linear regression on varying numbers of features in the four-rooms GridWorld, averaged over 10 runs; shaded areas represent 95\% confidence intervals.}
	\label{fig:valueBoundGap}
	\label{fig:stateFeatures}
  \vspace{-0.5cm}
\end{figure}

\subsection{State features}

In order to investigate the usefuleness of the representations produced by $\Pi U^{\pi}$, we construct state features directly by using the computed distances to project the states into a lower-dimensional space with the UMAP dimensionality reduction algorithm \citep{mcinnes18umap}\footnote{Note that since UMAP expects a metric, it is ill-defined with the diffuse metric $U^{\pi}$.}. We then apply linear regression of the true value function $V^{\pi}$ against the features to compute $\hat{V^{\pi}}$ and measure the average error across the state space. As baselines we compare against random features (RF), Proto Value Functions (PVF) \citep{mahadevan07pvf}, and the features produced by $\pi$-bisimulation \citep{castro2020scalable}. We present our results on the well-known four-rooms GridWorld \citep{sutton99between} 
in \autoref{fig:valueBoundGap} (right) and provide results on more environments in the appendix. Despite the independent couplings, $\Pi U^{\pi}$ performs on par with $\pi$-bisimulation, which optimizes over all couplings.

\section{Large-scale empirical evaluation}
\label{sec:empirical}
Having developed a greater understanding of the properties inherent to the representations produced by the MICo loss, we evaluate it on the Arcade Learning Environment \citep{bellemare13arcade}.
We added the MICo loss to all the JAX agents provided in the Dopamine library \citep{castro18dopamine}:    DQN \citep{mnih15human}, Rainbow \citep{hessel18rainbow}, QR-DQN \citep{dabney18qrdqn}, and IQN \citep{dabney18iqn}, using mean squared error loss to minimize $\mathcal{L}_{\text{TD}}$ for DQN (as suggested by \citet{obando2020revisiting}). Given the state-of-the-art results demonstrated by the Munchausen-IQN (M-IQN) agent \citep{vieillard20munchausen}, we also evaluated incorporating our loss into M-IQN.\footnote{Given that the authors of M-IQN had implemented their agent in TensorFlow (whereas our agents are in JAX), we have reimplemented M-IQN in JAX and run 5 independent runs (in contrast to the 3 run by \citet{vieillard20munchausen}.} For all experiments we used the hyperparameter settings provided with Dopamine. We found that a value of $\alpha=0.5$ worked well with quantile-based agents (QR-DQN, IQN, and M-IQN), while a value of $\alpha=0.01$ worked well with DQN and Rainbow. We hypothesise that the difference in scale of the quantile, categorical, and non-distributional loss functions concerned leads to these distinct values of $\alpha$ performing well. 
We found it important to use the Huber loss \citep{huber64robust} to minimize $\mathcal{L}_{\text{MICo}}$ as this emphasizes greater accuracy for smaller distances as oppoosed to larger distances. We experimented using the MSE loss but found that larger distances tended to overwhelm the optimization process, thereby degrading performance.

We evaluated on all 60 Atari 2600 games over 5 seeds and report the results in \autoref{fig:humanNormalizedIQM} (left), using the interquantile metric (IQM), proposed by
\citet{agarwal2021deep} as a more robust and reliable alternative to mean and median (which are reported in \autoref{fig:humanNormalizedJoined}).
The fact that the MICo loss provides consistent improvements over a wide range of baseline agents of varying complexity suggests that the MICo loss can help learn better representations for control. 

Additionally, we evaluated the MICo loss on twelve of the DM-Control suite from pixels environments \citep{tassa18dmcontrol}. As a base agent we used Soft Actor-Critic (SAC) \citep{haarnoja18sac} with the convolutional auto-encoder described by \citet{yarats19improving}. We applied the MICo loss on the output of the auto-encoder (with $\alpha=1e-5$) and maintained all other parameters untouched. Recently, \citet{zhang2021invariant} introduced DBC, which learns a dynamics and reward model on the output of the auto-encoder; their bisimulation loss uses the learned dynamics model in the computation of the Kantorovich distance between the next state transitions. We consider two variants of their algorithm: one which learns a stochastic dynamics model (DBC), and one which learns a deterministic dynamics model (DBC-Det). We replaced their bisimulation loss with the MICo loss (which, importantly, does not require a dynamics model) and kept all other parameters untouched. As \autoref{fig:humanNormalizedIQM} illustrates, the best performance is achieved with SAC augmented with the MICo loss; additionally, replacing the bisimulation loss of DBC with the MICo loss is able to recover the performance of DBC to match that of SAC.

Additional details and results are provided in the appendix.

\section{Related Work}

Bisimulation metrics were introduced for MDPs by \citet{Ferns04}, and have been extended in a number of directions  \citep{ferns05metrics,ferns06methods,taylor2008lax,taylor2009bounding,ferns2011bisimulation,comanici2012fly,bacci2013computing,bacci2013fly,abate2013approximation,ferns2014bisimulation,castro2020scalable}, with applications including policy transfer \citep{castro10using,santara2019extra}, representation learning \citep{ruan2015representation,comanici2015basis}, and state aggregation \citep{li2006towards}.

A range of other notions of similarity in MDPs have also been considered, such as action sequence equivalence \citep{givan2003equivalence}, temporally extended metrics \citep{amortila2019temporally}, MDP homomorphisms \citep{ravindran03sdmp}, utile distinction \citep{mccallum1996reinforcement}, and policy irrelevance \citep{jong2005state}, as well as notions of policy similarity \citep{pacchiano2020learning,moskovitz2020efficient}. \citet{li2006towards} review different notions of similarity applied to state aggregation. Recently, \citet{lelan21metrics} performed an exhaustive analysis of the continuity properties, relative to functions of interest in RL, of a number of existing metrics in the literature.

The notion of zero self-distance, central to the diffuse metrics defined in this paper, is increasingly encountered in machine learning applications involving approximation of losses. Of particular note is entropy-regularised optimal transport \citep{cuturi2013sinkhorn} and related quantities \citep{genevay2018learning,fatras2019learning,chizat2020faster,fatras2021minibatch}.

More broadly, many approaches to representation learning in deep RL have been considered, such as those based on auxiliary tasks (see e.g.~ \citep{sutton2011horde,jaderberg2016reinforcement,bellemare2017distributional,franccois2019combined,gelada2019deepmdp,guo2020bootstrap}), and other approaches such as successor features \citep{dayan93improving, barreto2016successor}.

\section{Conclusion}

In this paper, we have introduced the MICo distance, a notion of state similarity that can be learnt at scale and from samples. We have studied the theoretical properties of MICo, and proposed a new loss to make the non-zero self-distances of this diffuse metric compatible with function approximation, combining it with a variety of deep RL agents to obtain strong performance on the Arcade Learning Environment. In contrast to auxiliary losses that \emph{implicitly} shape an agent's representation, MICo directly modifies the features learnt by a deep RL agent; our results indicate that this helps improve performance. To the best of our knowledge, this is the first time {\em directly} shaping the representation of RL agents has been successfully applied at scale. We believe this represents an interesting new approach to representation learning in RL; continuing to develop theory, algorithms and implementations for direct representation shaping in deep RL is an important and promising direction for future work.

{\bf Broader impact statement}

This work lies in the realm of ``foundational RL'' in that
it contributes to the fundamental understanding and
development of reinforcement learning algorithms and
theory. As such, despite us agreeing in the importance
of this discussion, our work is quite far removed from
ethical issues and potential societal consequences.

\section{Acknowledgements}
The authors would like to thank Gheorghe Comanici, Rishabh Agarwal, Nino Vieillard, and Matthieu Geist for their valuable feedback on the paper and experiments. Pablo Samuel Castro would like to thank Roman Novak and Jascha Sohl-Dickstein for their help in getting angular distances to work stably! Thanks to Hongyu Zang for pointing out that the x-axis labels for the SAC experiments needed to be fixed. Finally, the authors would like to thank the reviewers (both ICML'21 and NeurIPS'21) for helping make this paper better.

\clearpage

\bibliographystyle{plainnat}
\bibliography{metrics}

\clearpage
\onecolumn

\begin{appendix}
	
\section*{\centering Supplementary Material: \\ MICo: Improved representations via sampling-based state similarity for Markov decision processes}

\section{Extended background material}
\label{sec:extendedBackground}
In this section we provide a more extensive background review.

\subsection{Reinforcement learning}

In this section we give a slightly more expansive overview of relevant key concepts in reinforcement learning, without the space constraints of the main paper.
Denoting by $\mathscr{P}(S)$ the set of probability distributions on a set $S$, we define a Markov decision process $(\mathcal{X}, \mathcal{A}, \gamma, P,
r)$ as:
\begin{itemize}[leftmargin=0.5cm,topsep=0pt,itemsep=0pt]
    \item A finite state space $\mathcal{X}$;
    \item A finite action space $\mathcal{A}$;
    \item A transition kernel $P : \mathcal{X} \times \mathcal{A}\rightarrow \mathscr{P}(\cX)$;
    \item A reward function $r : \mathcal{X} \times\mathcal{A} \rightarrow \mathbb{R}$;
    \item A discount factor $\gamma \in [0,1)$.
\end{itemize}
For notational convenience we introduce the notation $P_x^a \in \mathscr{P}(\mathcal{X})$ for the next-state distribution given state-action pair $(x, a)$, and $r_x^a$ for the corresponding immediate reward.

Policies are mappings from states to distributions
over actions: $\pi \in \mathscr{P}(\mathcal{A})^\mathcal{X}$ and induce a {\em value
function} $V^{\pi}:\cX\rightarrow\mathbb{R}$ defined via the recurrence:
\[ V^{\pi}(x) := \mathbb{E}_{a\sim\pi(x)}\left[ r_x^a + \gamma\mathbb{E}_{x'\sim
P_x^a}[V^{\pi}(x')]\right] \, . \]
It can be shown that this recurrence uniquely defines $V^\pi$ through a contraction mapping argument \citep{bertsekas1996neuro}.

The control problem is concerned with finding the optimal
policy
\[ \pi^* = \arg\max_{\pi\in\mathscr{P}(\mathcal{A})^\mathcal{X}}V^{\pi} \, . \]
It can be shown that while the optimisation problem above appears to have multiple objectives (one for each coordinate of $V^\pi$, there is in fact a policy $\pi^* \in \mathscr{P}(\mathcal{A})^\mathcal{X}$ that simultaneously maximises all coordinates of $V^\pi$, and that this policy can be taken to be deterministic; that is, for each $x \in \mathcal{X}$, $\pi(\cdot|x) \in \mathscr{P}(\mathcal{A})$ attributes probability 1 to a single action. 
In reinforcement learning in particular, we are often interested in finding, or approximating, $\pi^*$ from direct interaction with the MDP in question via sample trajectories, {\em without knowledge of $P$ or $r$} (and sometimes not even $\cX$).

\subsection{Metrics}\label{sec:app:metrics}
A {\em metric} $d$ on a set $X$ is a function $d:X\times X\rightarrow [0, \infty)$ respecting the following axioms for any $x, y, z \in X$:
\begin{enumerate}[leftmargin=0.5cm,topsep=0pt,itemsep=0pt]
    \item {\bf Identity of indiscernibles: } $d(x, y) = 0 \iff x = y$;
    \item {\bf Symmetry: } $d(x, y) = d(y, x)$;
    \item {\bf Triangle inequality: } $d(x, y) \leq d(x, z) + d(z, y)$.
\end{enumerate}

A {\em pseudometric} is similar, but the "identity of indiscernibles" axiom is weakened:
\begin{enumerate}[leftmargin=0.5cm,topsep=0pt,itemsep=0pt]
    \item $x = y \implies d(x, y) = 0$;
    \item $d(x, y) = d(y, x)$;
    \item $d(x, y) \leq d(x, z) + d(z, y)$.
\end{enumerate}
Note that the weakened first condition {\em does} allow one to have $d(x, y) = 0$ when $x\ne y$.

A {\em (pseudo)metric space} $(X, d)$ is defined as a set $X$ together with a (pseudo)metric $d$ defined on $X$.

\subsection{State similarity and bisimulation metrics}

Bisimulation is a fundamental notion of behavioural equivalence introduced by
Park and Milner \citep{Milner89} in the early 1980s in the context of
nondeterministic transition systems.  The probabilistic analogue was introduced
by \citet{Larsen91}.  The notion of an equivalence relation is not suitable to
capture the extent to which quantitative systems may resemble each other in
behaviour.  To provide a quantitative notion, bisimulation metrics were
introduced by \citet{Desharnais99b,Desharnais04} in the context of probabilistic
transition systems without rewards.  In reinforcement learning the reward is an
important ingredient, accordingly the \emph{bisimulation metric} for states of MDPs was
introduced by \citet{Ferns04}. Much work followed this initial introduction of bisimulation metrics into RL, as described in the main paper. We briefly reviewed the bisimulation metric in Section~\ref{sec:background}, and now provide additional detail around some of the key associated mathematical concepts.

Central to the definition of the bisimulation metric is the operator $T_k : \mathcal{M}(\mathcal{X}) \rightarrow \mathcal{M}(\mathcal{X})$, defined over $\mathcal{M}(\mathcal{X})$, the space of pseudometrics on $\mathcal{X}$. Pseudometrics were explored in more detail in Section~\ref{sec:app:metrics}. We now turn to the definition of the operator itself, given by
\begin{align*}
    T_k(d)(x, y) = \max_{a \in \mathcal{A}} [|r_x^a - r_y^a] + \gamma W_d(P^a_x, P^a_y)] \, ,
\end{align*}
for each $d \in \mathcal{M}(\mathcal{X})$, and each $x, y \in \mathcal{X}$. It can be verified that the function $T_K(d) : \mathcal{X} \times \mathcal{X} \rightarrow \mathbb{R}$ satisfies the properties of a pseudometric, so under this definition $T_K$ does indeed map $\mathcal{M}(\mathcal{X})$ into itself.

The other central mathematical concept underpinning the operator $T_K$ is the Wasserstein distance $W_d$ using base metric $d$. $W_d$ is formally a pseudometric over the set of probability distributions $\mathscr{P}(\mathcal{X})$, defined as the solution to an optimisation problem. The problem specifically is formulated as finding an optimal coupling between the two input probability distributions that minimises a notion of transport cost associated with $d$. Mathematically, for two probability distributions $\mu, \mu' \in \mathscr{P}(\mathcal{X})$, we have
\begin{align*}
    W_d(\mu, \mu') = \min_{\substack{(Z, Z') \\ Z \sim \mu, Z' \sim \nu'}} \mathbb{E}[d(Z, Z')] \, .
\end{align*}
Note that the pair of random variables $(Z, Z')$ attaining the minimum in the above expression will in general not be independent. That the minimum is actually attained in the above example in the case of a finite set $\mathcal{X}$ can be seen by expressing the optimisation problem as a linear program. Minima are obtained in much more general settings too; see \citet{villani08optimal}.

Finally, the operator $T_K$ can be analysed in a similar way to standard operators in dynamic programming for reinforcement learning. It can be shown that it is a contraction mapping with respect to the $L^\infty$ metric over $\mathcal{M}(\mathcal{X})$, and that $\mathcal{M}(\mathcal{X})$ is a complete metric space with respect to the same metric \citep{ferns2011bisimulation}. Thus, by Banach's fixed point theorem, $T_K$ has a unique fixed point in $\mathcal{M}(\mathcal{X})$, and repeated application of $T_K$ to any initial pseudometric will converge to this fixed point.

\subsection{Further details on diffuse and partial metrics}

The notion of a distance function having non-zero self distance was first
introduced by \citet{matthews_pm} who called it a \emph{partial metric}.  We
define it below:
\begin{definition}
	Given a set $\mathcal{X}$, a function $d:\mathcal{X}\times 
	\mathcal{X} \to \mathbb{R}$ is a partial metric if the following axioms hold:
	(i) $x=y \iff d(x,x)=d(y,y)=d(x,y)$ for any $x,y\in \mathcal{X}$;
	(ii) $d(x,x)\leq d(y,x)$ for any $x,y\in \mathcal{X}$;
	(iii) $d(x,y)= d(y,x)$ for any $x,y\in \mathcal{X}$;
	(iv) $d(x,y)\leq d(x,z)+d(y,z)-d(z, z)$ $\forall x,y,z\in \mathcal{X}$.
\end{definition}
This definition was introduced to recover a proper metric from the distance
function: that is, given a partial metric $d$, one is guaranteed that
$\tilde{d}(x,y)=d(x,y)-\frac{1}{2}\left(d(x,x)+d(y,y)\right)$ is a proper
metric.

The above definition is still too stringent for the Łukaszyk–Karmowski distance (and hence MICo distance), since it fails axiom 4 as shown in the following counterexample.

\begin{example}
	The Łukaszyk–Karmowski   distance does not satisfy the modified triangle inequality: let $\mathcal{X}$ be $[0,1]$, and $\rho$ be the Euclidean distance $|\cdot|$. Let $\mu$,$\nu$ be Dirac measures concentrated at 0 and 1, and let $\eta$ be $\frac{1}{2}(\delta_0+\delta_1)$. Then one can calculate that $d_{LK}(\rho)(\mu,\nu)=1$, while $d_{LK}(\rho)(\mu,\eta)+d_{LK}(\rho)(\nu,\eta)-d_{LK}(\rho)(\eta,\eta)=1/2$, breaking the inequality.
\end{example}

This naturally leads us to the notion of diffuse metrics defined in the main paper.

\section{Proof of Proposition 4.2}
\label{sec:missingProofs}

\micoContraction*

\begin{proof}
	Let $U, U' \in \mathbb{R}^{\mathcal{X}\times\mathcal{X}}$. Then note that
  \small
	\begin{align*}
    |(T^\pi U)(x, y) - (T^\pi U')(x, y)|  = \left|\gamma\sum_{x', y'}\pi(a|x)\pi(b|y)P_x^a(x')P_y^b(y') (U - U')(x', y') \right| \leq \gamma \|U - U'\|_\infty \, .
	\end{align*}
  \normalsize
	for any $x,y \in \mathcal{X}$, as required.
\end{proof}

\section{Experimental details}

We will first describe the regular network and training setup for these agents so as to facilitate the description of our loss.

\subsection{Baseline network and loss description}

The networks used by Dopamine for the ALE consist of 3 convolutional layers followed by two fully-connected layers (the output of the networks depends on the agent). We denote the output of the convolutional layers by $\phi_{\omega}$ with parameters $\omega$, and the remaining fully connected layers by $\psi_{\xi}$ with parameters $\xi$. Thus, given an input state $x$ (e.g. a stack of 4 Atari frames), the output of the network is $Q_{\xi,\omega}(x, \cdot) = \psi_{\xi}(\phi_{\omega}(x))$. Two copies of this network are maintained: an {\em online} network and a {\em target} network; we will denote the parameters of the target network by $\bar{\xi}$ and $\bar{\omega}$. During learning, the parameters of the online network are updated every 4 environment steps, while the target network parameters are synced with the online network parameters every 8000 environment steps.
We refer to the loss used by the various agents considered as $\cL_{\text{TD}}$; for example, for DQN this would be:
\[ \cL_{\text{TD}}(\xi, \omega) := \mathbb{E}_{(x, a, r, x')\sim\cD}\left[\rho\left(r + \gamma\max_{a'\in\cA}Q_{\bar{\xi},\bar{\omega}}(x', a') - Q_{\xi, \omega}(x, a) \right) \right] \, , \]
where $\cD$ is a replay buffer with a capacity of 1M transitions, and $\rho$ is the Huber loss.

\subsection{MICo loss description}

We will be applying the MICo loss to $\phi_{\omega}(x)$. As described in \autoref{sec:loss}, we express the distance between two states as:
\[ U_{\omega}(x, y) = \frac{\| \phi_{\omega}(x) \|^2_2 + \| \phi_{\bar{\omega}}(y) \|^2_2 }{2} + \beta \theta(\phi_{\omega}(x), \phi_{\bar{\omega}}(y)) \, , \]
where $\theta(\phi_\omega(x), \phi_{\bar{\omega}}(y))$ is the angle between vectors $\phi_\omega(x)$ and $\phi_{\bar{\omega}}(y)$ and $\beta$ is a scalar.
Note that we are using the target network for the $y$ representations; this was done for learning stability. We used $\beta=0.1$ for the results in the main paper, but present some results with different values of $\beta$ below.

In order to get a numerically stable operation, we implement the angular distance between representations $\phi_\omega(x)$ and $\phi_\omega(y)$ according to the calculations
\begin{align*}
	\text{CS}(\phi_\omega(x), \phi_\omega(y)) & = \frac{\langle \phi_\omega(x), \phi_\omega(y)\rangle }{\|\phi_\omega(x)\|\|\phi_\omega(y)\|} \\
	\theta(\phi_\omega(x), \phi_\omega(y)) & = \arctantwo\left(\sqrt{1 - \text{CS}(\phi_\omega(x), \phi_\omega(y))^2}, \text{CS}(\phi_\omega(x), \phi_\omega(y))\right) \, .
\end{align*}

Based on \autoref{eqn:uUpdate}, our learning target is then (note the target network is used for both representations here):
\[ T^U_{\bar{\omega}}(r_x, x', r_y, y') = |r_x - r_y| + \gamma U_{\bar{\omega}}(x', y') \, , \]
and the loss is
\begin{align*}
  \mathcal{L}_{\text{MICo}}(\omega) = & \mathbb{E}_{\begin{subarray}{l}\langle x, r_x, x'\rangle \\  \langle y, r_y, y'\rangle \end{subarray}\sim\cD}\left[ \left(T^U_{\bar{\omega}}(r_x, x', r_y, y') - U_{\omega}(x, y)\right)^2 \right] \, ,
\end{align*}

As mentioned in \autoref{sec:loss}, we use the same mini-batch sampled for $\cL_{\text{TD}}$ for computing $\cL_{\text{MICo}}$. Specifically, we follow the method introduced by \citet{castro2020scalable} for constructing new matrices that allow us to compute the distances between all pairs of sampled states (see code for details on matrix operations). 
Our combined loss is then $\cL_{\alpha}(\xi, \omega) = (1-\alpha)\mathcal{L}_{\text{TD}}(\xi, \omega) + \alpha\mathcal{L}_{\text{MICo}}(\omega)$.

\subsection{Hyperparameters for soft actor-critic}
We re-implemented the DBC algorithm from \citet{zhang2021invariant} on top of the Soft Actor-Critic algorithm \citep{haarnoja18sac}  provided by the Dopamine library \citep{castro18dopamine}. We compared the following algorithms, using the same hyperparameters for all\footnote{See https://github.com/google-research/google-research/tree/master/mico for all hyperparameter settings.}:

\begin{enumerate}[leftmargin=0.5cm]
  \item {\bf SAC:} This is Soft Actor-Critic \citep{haarnoja18sac} with the convolutional encoder described by \citet{yarats19improving}.
  \item {\bf DBC:} This is DBC, based on SAC, as described by \citet{zhang2021invariant}.
  \item {\bf DBC-Det:} In the code provided by \citet{zhang2021invariant}, the default setting was to assume deterministic transitions (which is an easier dynamics model to learn), so we decided to compare against this version as well. It is interesting to note that the performance is roughly the same as for DBC.
  \item {\bf MICo:} This a modified version of SAC, adding the MICo loss to the output of the encoder. Note that the encoder output is the same one used by DBC for their dynamics and reward models.
  \item {\bf DBC+MICo:} Instead of using the bisimulation loss of \citet{zhang2021invariant}, which relies on the learned dynamics model, we use our MICo loss. We kept all other components untouched (so a dynamics and reward model were still being learned).
\end{enumerate}

It is worth noting that some of the hyperparameters we used differ from those listed in the code provided by \citet{zhang2021invariant}; in our experiments they hyperparameters for all agents are based on the default SAC hyperparameters in the Dopamine library \citep{castro18dopamine}.

For the ALE experiments we used the ``squaring'' of the sampled batches introduced by \citet{castro2020scalable} (where all pairs of sampled states are considered). However, the implementation provided by \citet{zhang2021invariant} instead created a copy of the sampled batch of transitions and shuffled them; we chose to follow this setup for the SAC-based experiments. Thus, while in the ALE experiments we are comparing $m^2$ pairs of states (where $m$ is the batch size) at each training step, in the SAC-based experiments we are only comparing $m$ pairs of states.

The aggregate results are displayed in \autoref{fig:humanNormalizedIQM}, and per-environment results in \autoref{fig:sacAllEnvs} below.

\section{Additional experimental results}
\subsection{Additional state feature results}
The results shown in \autoref{fig:stateFeatures} are on the well-known four-rooms GridWorld \citep{sutton99between}. We provide extra experiments in \autoref{fig:stateFeatures2}.

\begin{figure}[!h]
	\centering
	\includegraphics[keepaspectratio,width=.48\textwidth]{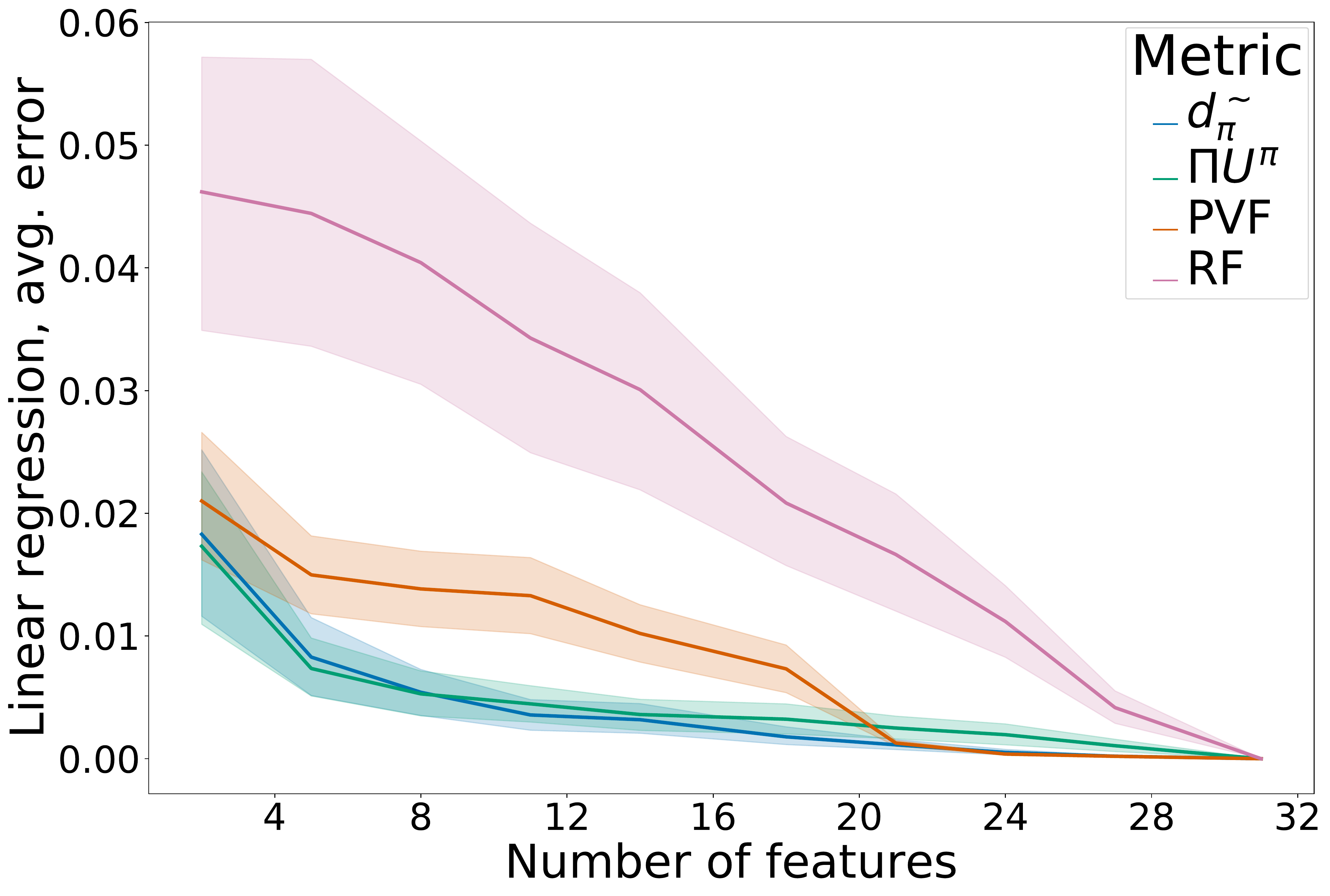}
	\includegraphics[keepaspectratio,width=.48\textwidth]{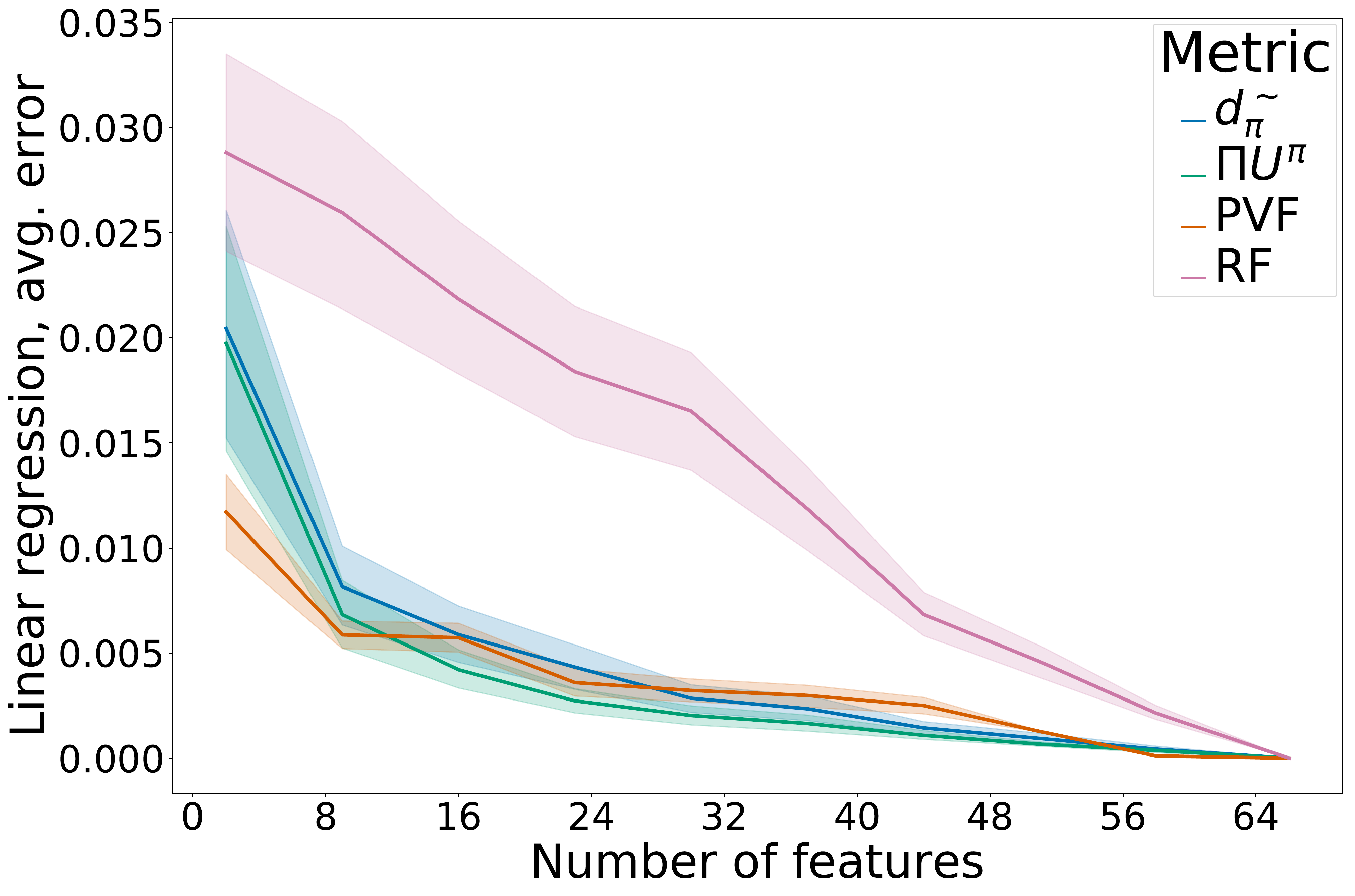}
  \caption{Average error when performing linear regression on varying numbers of features on the mirrored rooms introduced by \citet{castro2020scalable} (left) and the grid task introduced by \citet{dayan93improving} (right). Averaged over 10 runs; shaded areas represent 95\% confidence intervals.}
  \label{fig:stateFeatures2}
\end{figure}

\subsection{Complete ALE experiments}
\begin{figure*}[!t]
	\centering
	\includegraphics[keepaspectratio,width=\textwidth]{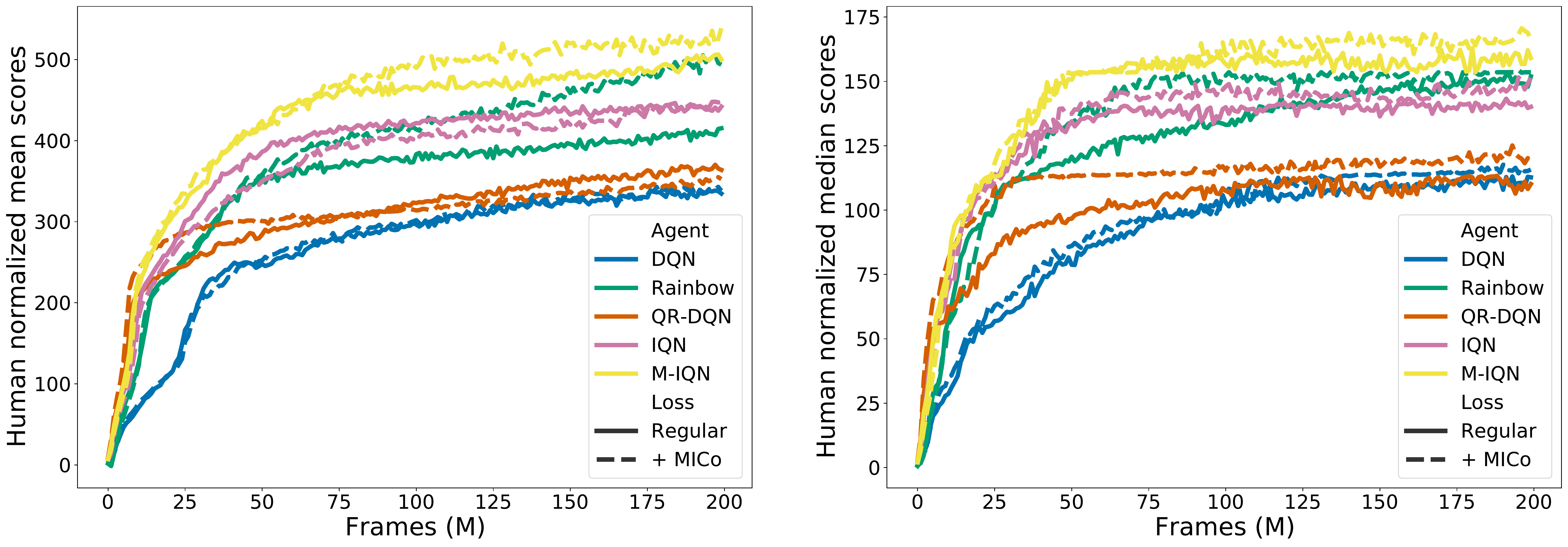}
  \caption{Mean (left) and median (right) human normalized scores across 60 Atari 2600 games, averaged over 5 independent runs.}
  \vspace{-0.5cm}
	\label{fig:humanNormalizedJoined}
\end{figure*}

We additionally provide complete results for all the agents in \autoref{fig:dqn_all_games}, \autoref{fig:rainbow_all_games}, \autoref{fig:quantile_all_games}, \autoref{fig:iqn_all_games}, and \autoref{fig:miqn_all_games}.

\begin{figure*}[!h]
  \centering
	\includegraphics[keepaspectratio,width=\textwidth]{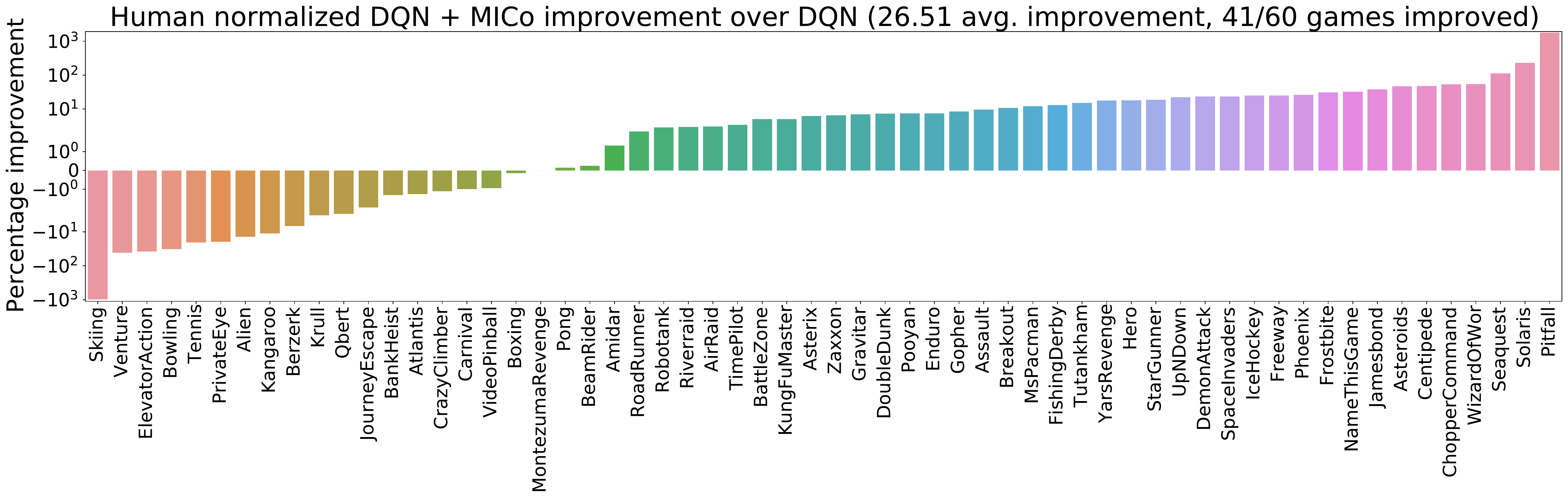}
	\includegraphics[keepaspectratio,width=\textwidth]{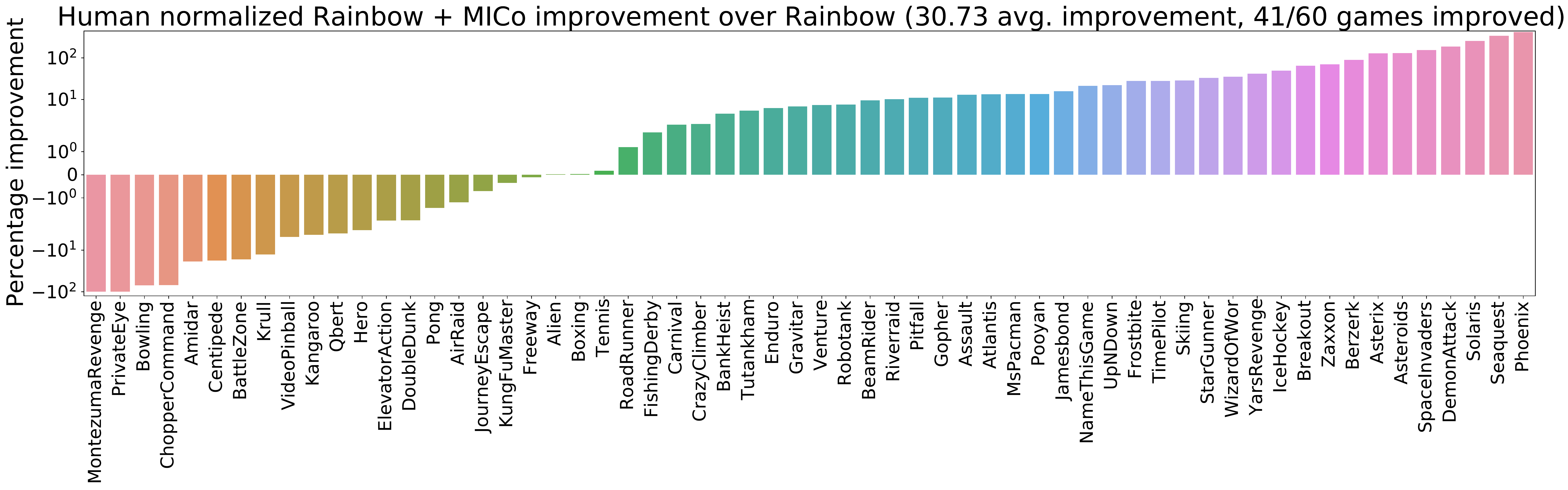}
	\includegraphics[keepaspectratio,width=\textwidth]{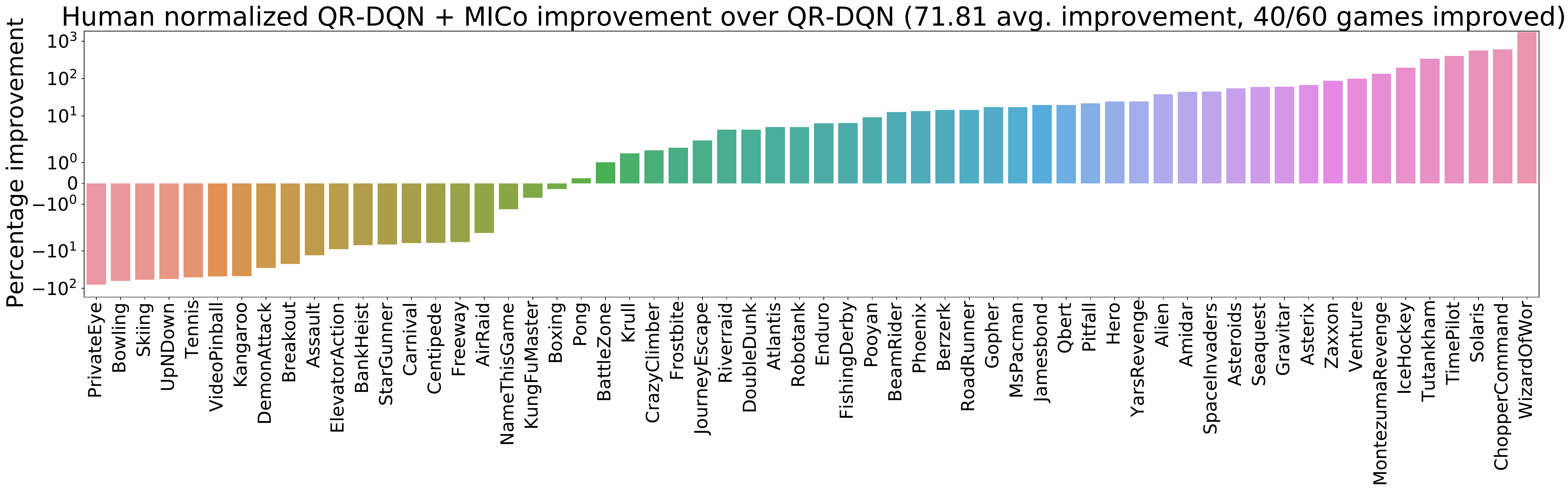}
	\includegraphics[keepaspectratio,width=\textwidth]{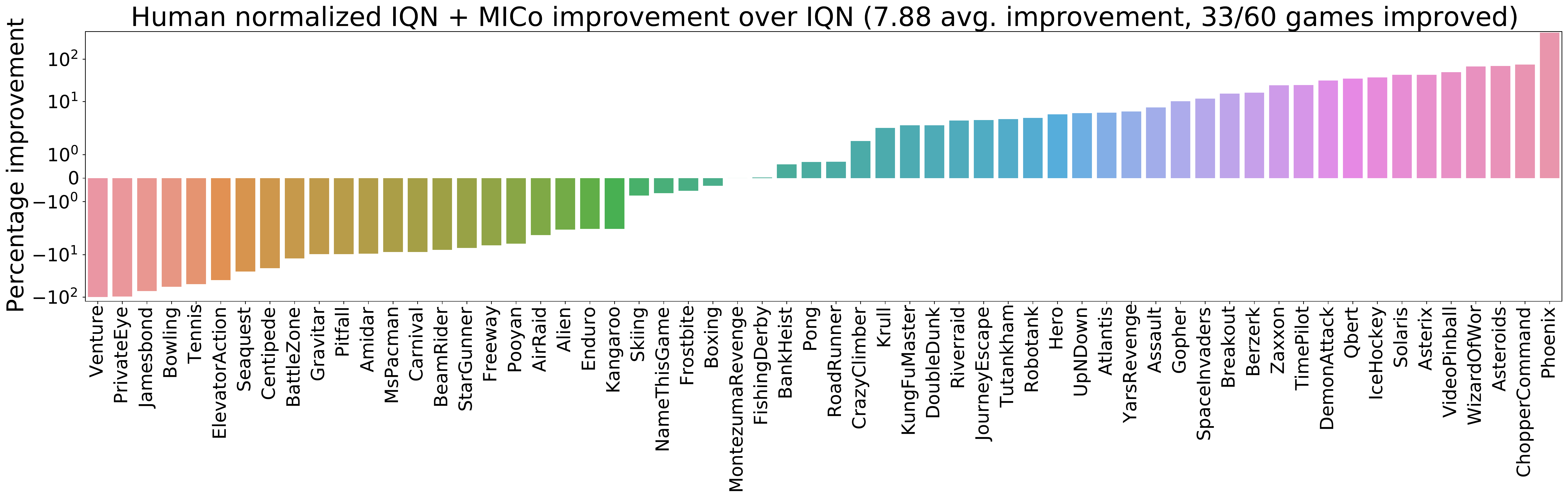}
  \includegraphics[keepaspectratio,width=\textwidth]{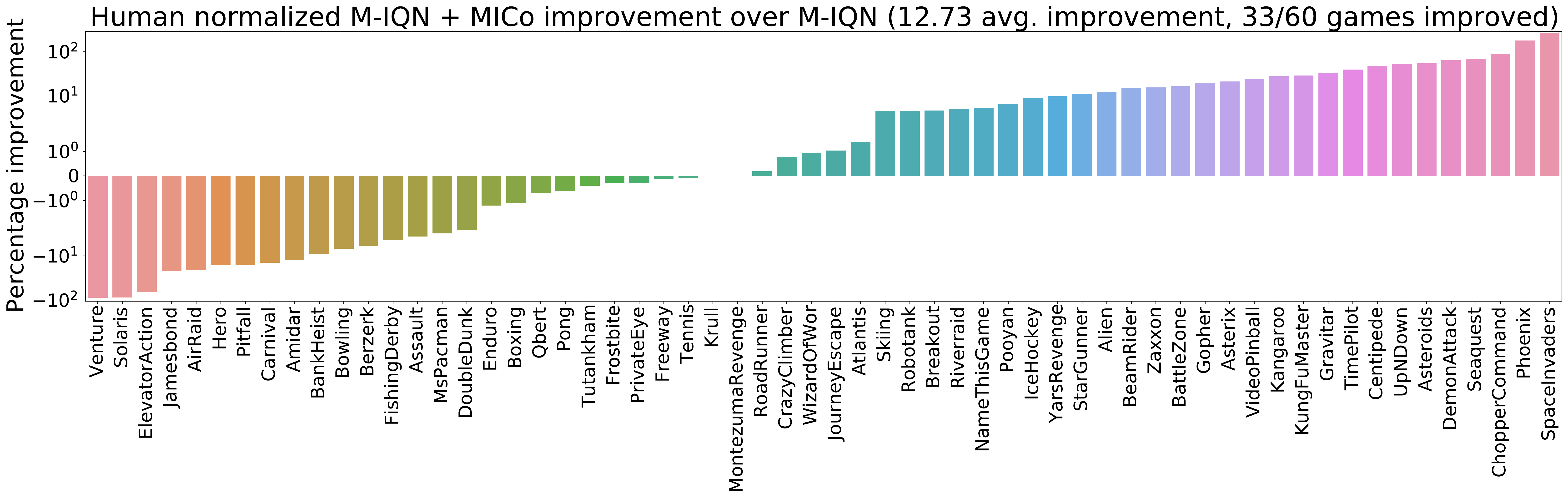}
  \caption{From top to bottom, percentage improvement in returns (averaged over the last 5 iterations) when adding $\mathcal{L}_{\text{MICo}}$ to DQN, Rainbow, QR-DQN, IQN, and M-DQN. The results for are averaged over 5 independent runs.}
  \label{fig:atariBarPlots}
\end{figure*}

\begin{figure*}[!h]
	\centering
	\includegraphics[keepaspectratio,width=0.95\textwidth]{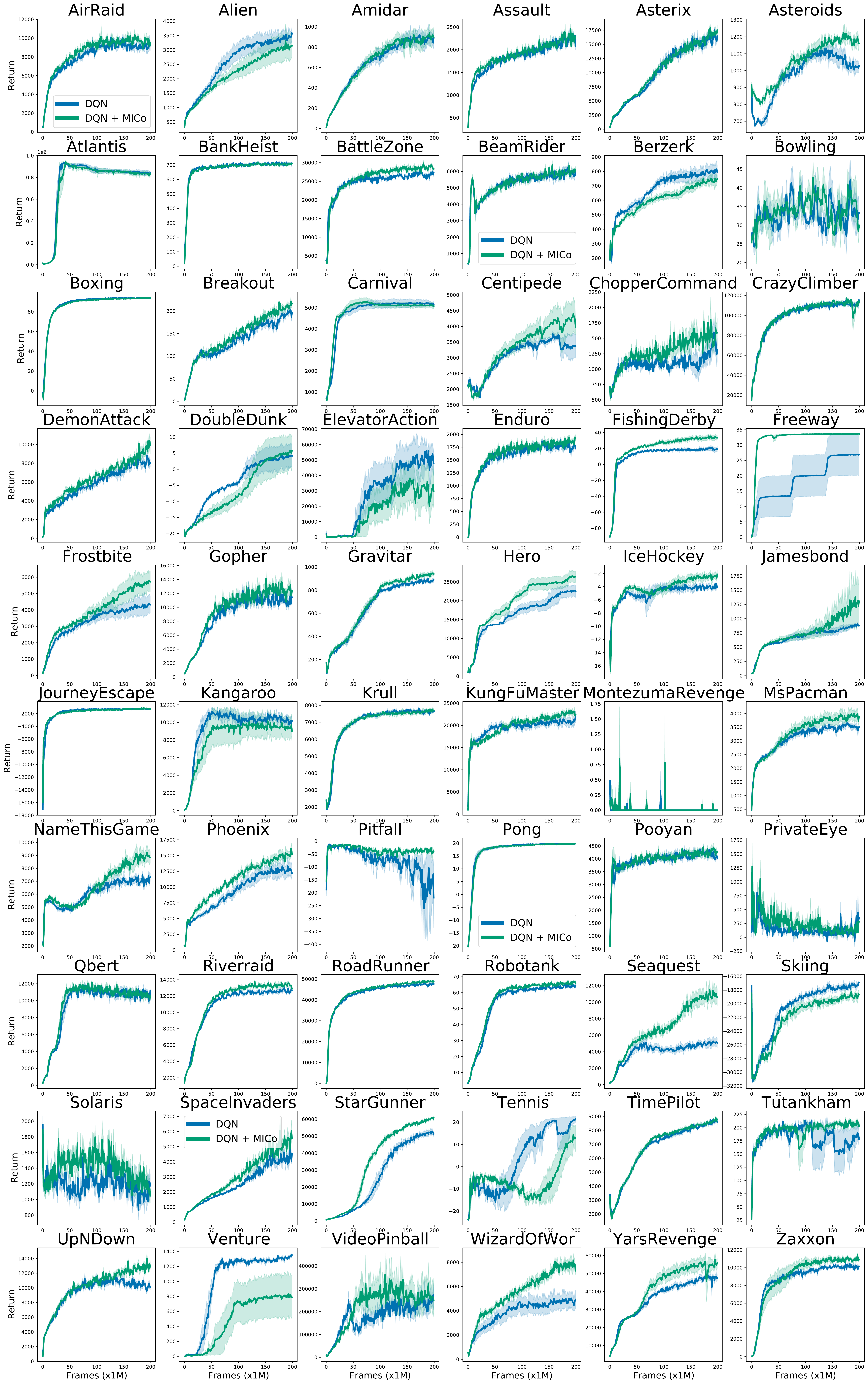}
	\caption{Training curves for DQN agents. The results for all games and agents are over 5 independent runs, and shaded regions report 75\% confidence intervals.}
	\label{fig:dqn_all_games}
\end{figure*}

\begin{figure*}[!h]
	\centering
	\includegraphics[keepaspectratio,width=0.95\textwidth]{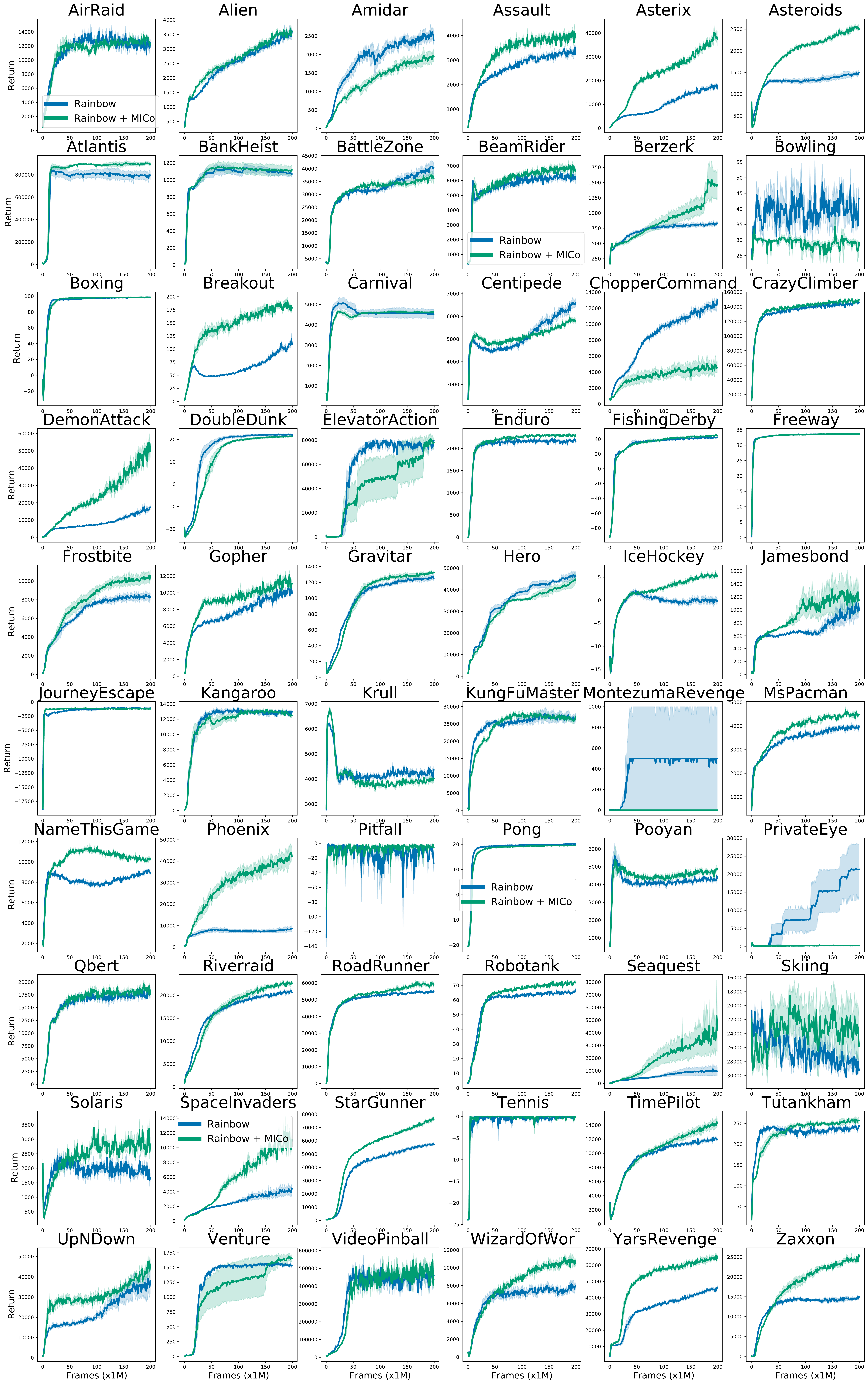}
	\caption{Training curves for Rainbow agents. The results for all games and agents are over 5 independent runs, and shaded regions report 75\% confidence intervals.}
	\label{fig:rainbow_all_games}
\end{figure*}

\begin{figure*}[!h]
	\centering
	\includegraphics[keepaspectratio,width=0.95\textwidth]{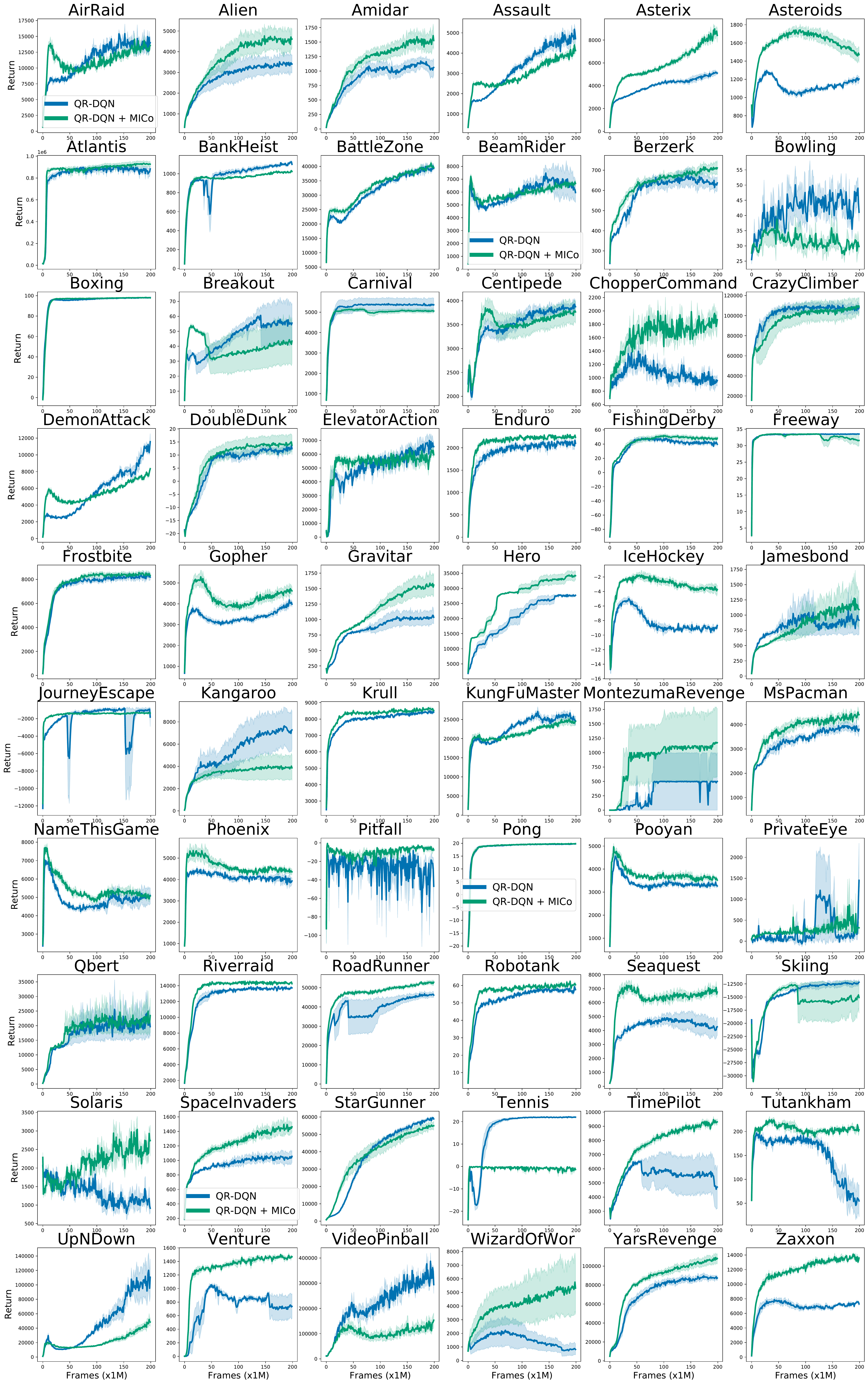}
	\caption{Training curves for QR-DQN agents. The results for all games and agents are over 5 independent runs, and shaded regions report 75\% confidence intervals.}
	\label{fig:quantile_all_games}
\end{figure*}

\begin{figure*}[!h]
	\centering
	\includegraphics[keepaspectratio,width=0.95\textwidth]{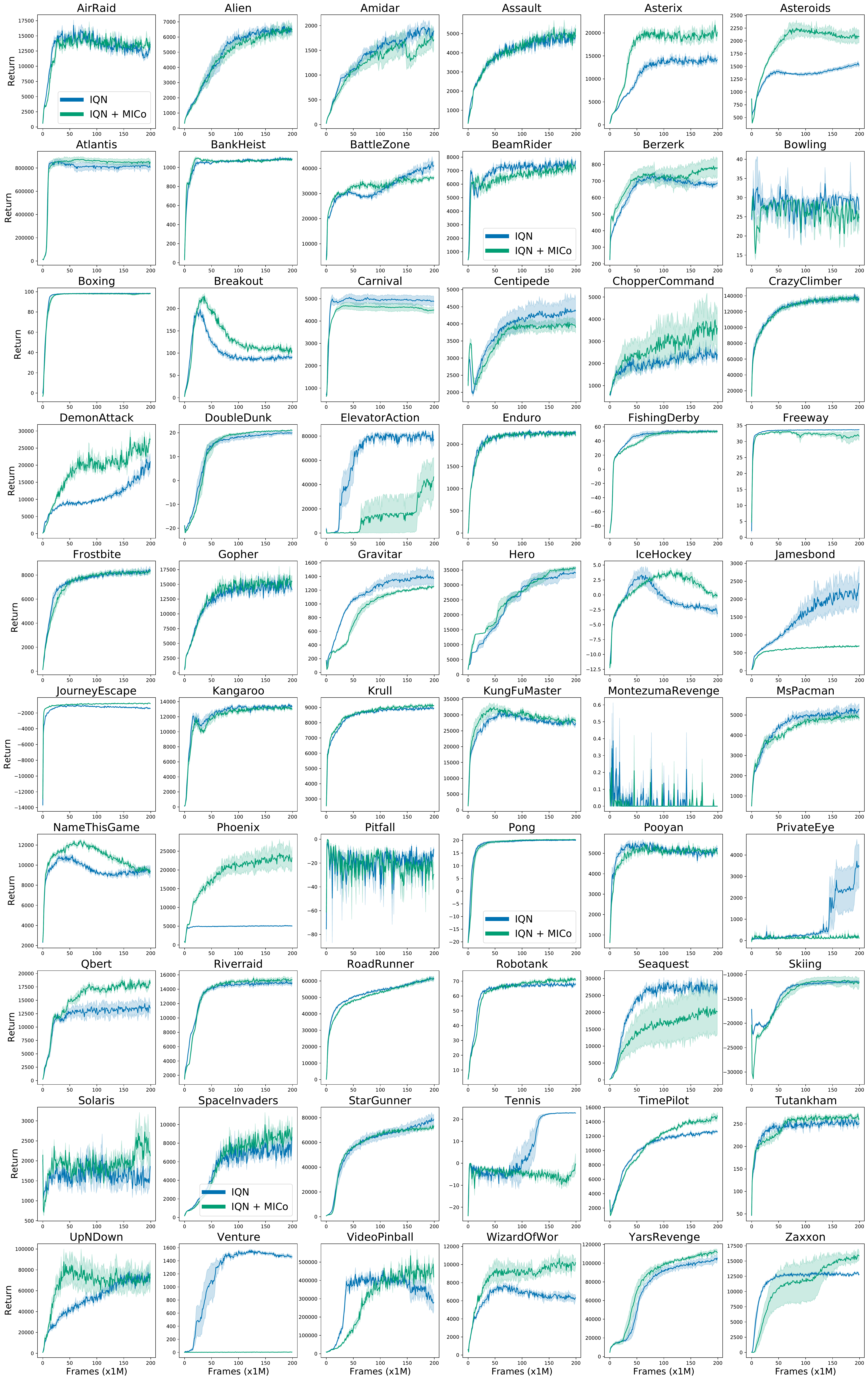}
	\caption{Training curves for IQN agents. The results for all games and agents are over 5 independent runs, and shaded regions report 75\% confidence intervals.}
	\label{fig:iqn_all_games}
\end{figure*}

\begin{figure*}[!h]
	\centering
	\includegraphics[keepaspectratio,width=0.95\textwidth]{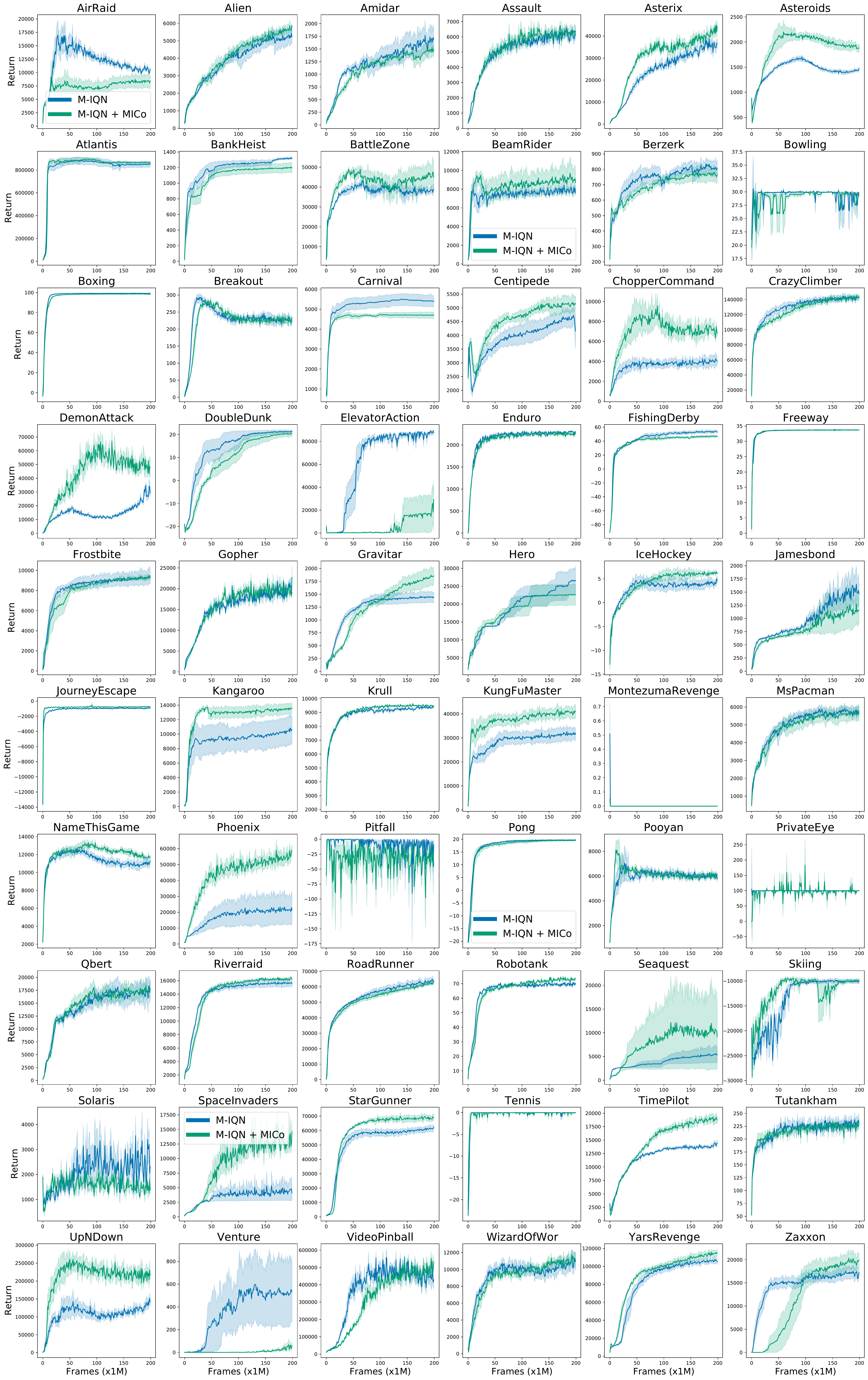}
	\caption{Training curves for M-IQN agents. The results for all games and agents are over 5 independent runs, and shaded regions report 75\% confidence intervals.}
	\label{fig:miqn_all_games}
\end{figure*}

\subsection{Sweep over $\alpha$ and $\beta$ values}

In \autoref{fig:sweep} we demonstrate the performance of the MICo loss when
added to Rainbow over a number of different values of $\alpha$ and $\beta$.
For each agent, we ran a similar hyperparameter sweep over $\alpha$ and $\beta$
on the same six games displayed in \autoref{fig:sweep} to determine settings to
be used in the full ALE experiments.

\begin{figure*}[!h]
	\centering
	\includegraphics[keepaspectratio,width=0.95\textwidth]{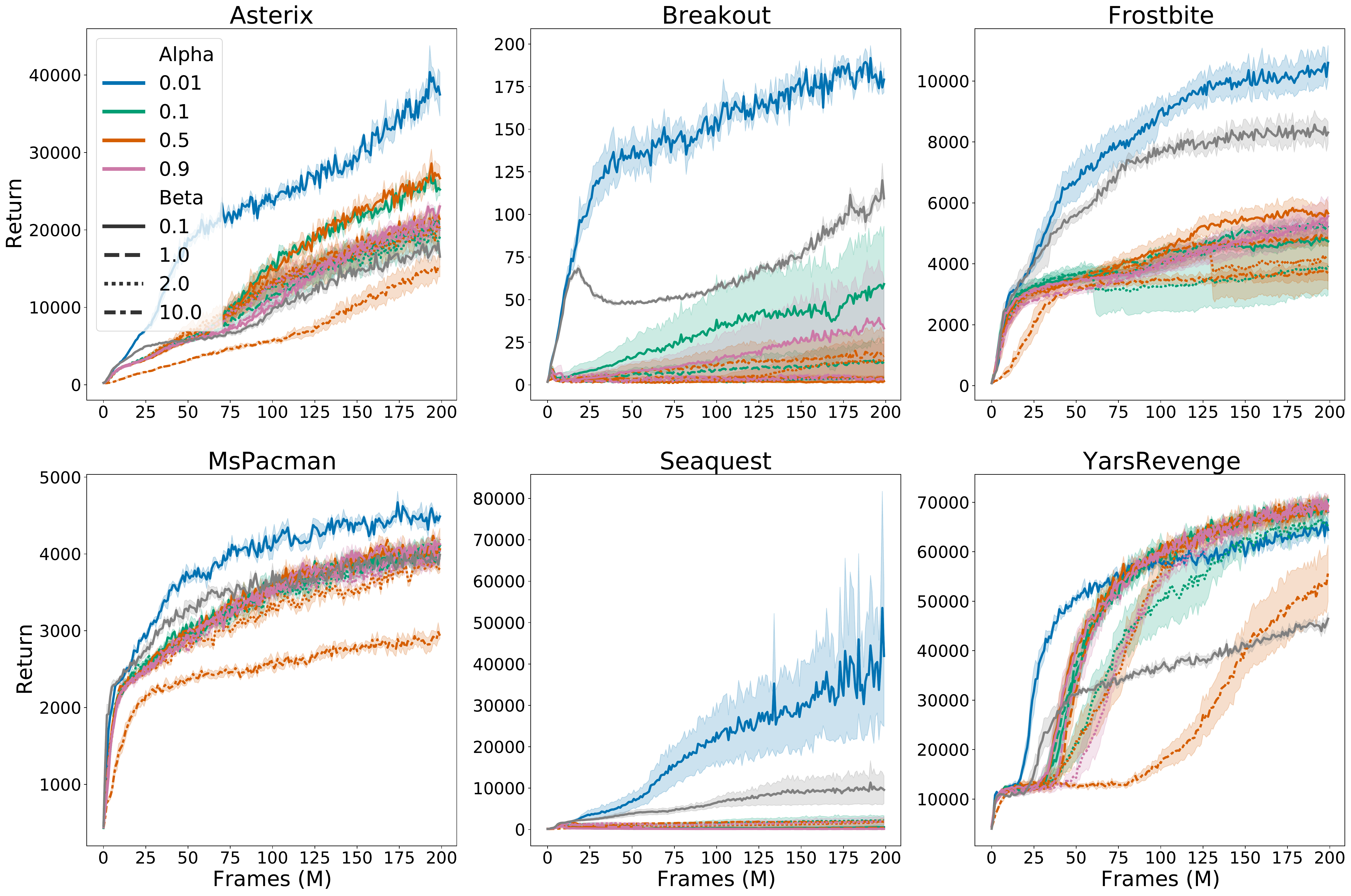}
	\caption{Sweeping over various values of $\alpha$ and $\beta$ when adding the MICo loss to Rainbow. The grey line represents regular Rainbow.}
	\label{fig:sweep}
\end{figure*}

\subsection{Complete DM-Control results}

Full per-environment results are provided in \autoref{fig:sacAllEnvs}.

\begin{figure*}[!h]
	\centering
	\includegraphics[keepaspectratio,width=0.95\textwidth]{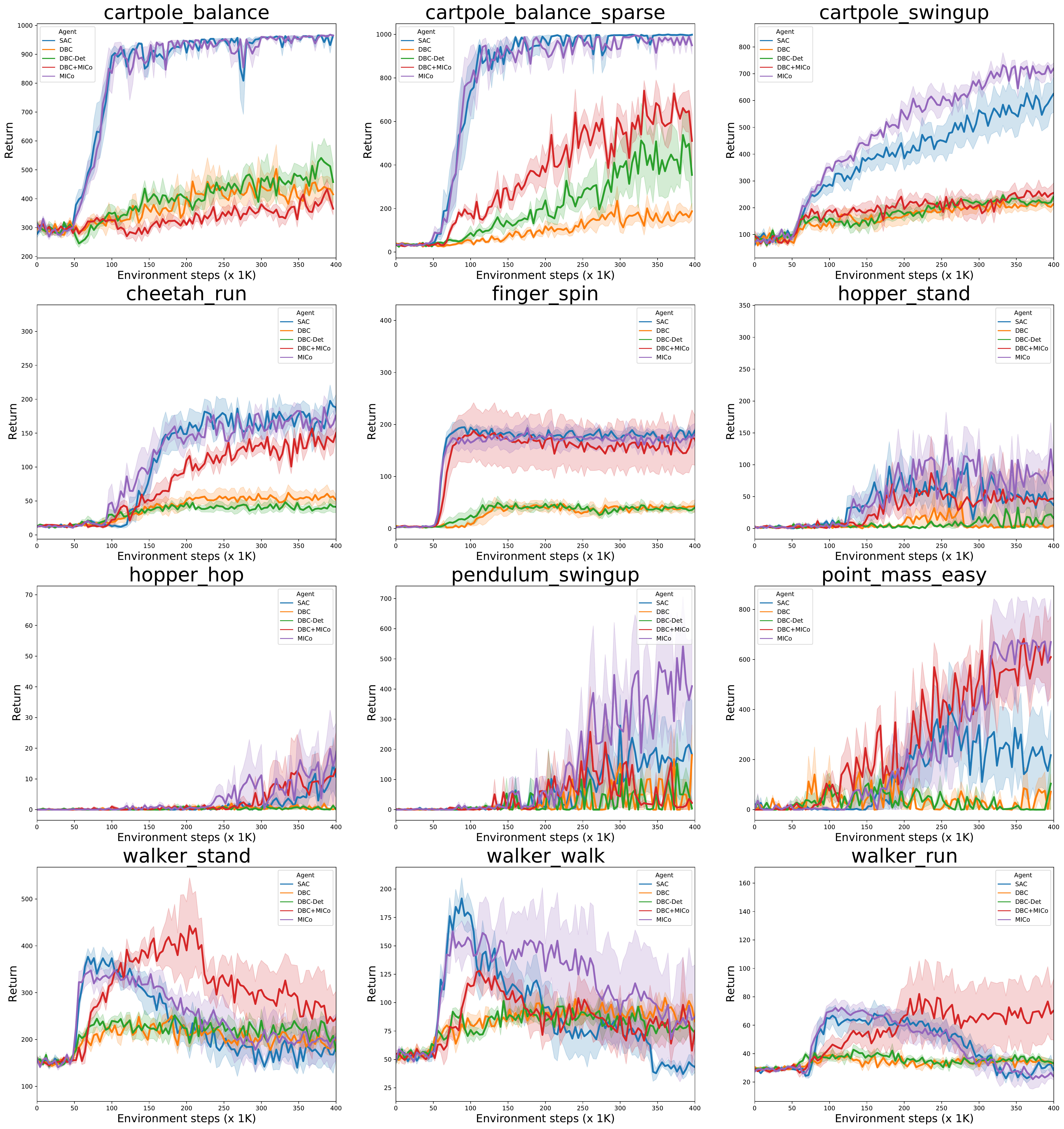}
  \caption{Comparison of all agents on twelve of the DM-Control suite. Each algorithm and environment was run for five independent seeds, and the shaded areas report 75\% confidence intervals.}
	\label{fig:sacAllEnvs}
\end{figure*}

\subsection{Compute time and infrastructure}
For \autoref{fig:valueBoundGap} each run took approximately 10 minutes. For \autoref{fig:stateFeatures} and \autoref{fig:stateFeatures2} the running time varied for each environment and per metric but a conservative estimate is 30 minutes per run. All GPU experiments were run on NVIDIA Tesla P100 GPUs. Each Atari game takes approximately 5 days (300 hours) to run for 200M frames.

\end{appendix}

\end{document}